 \documentclass[sigconf]{acmart}
\AtBeginDocument{%
  }

\newcommand{\mytheoremname}{\bfseries Theorem}
\newcommand{\mylemmaname}{\bfseries Lemma}
\newcommand{\mydefinitionname}{\bfseries Definition}
\newcommand{\mypropositionname}{\bfseries Proposition}
\newcommand{\myremarkname}{\bfseries Remark}

\newtheorem{theorem}{\mytheoremname}
\newtheorem{lemma}[theorem]{\mylemmaname}
\newtheorem{definition}[theorem]{\mydefinitionname} 
\newtheorem{proposition}[theorem]{\mypropositionname}
\newtheorem{remark}[theorem]{\myremarkname}

\usepackage{hyperref}
\newcommand{\zwc}[1]{{\color{black} #1}}

\usepackage[
type={CC},
modifier={by-nc},
version={4.0},
]{doclicense}
\usepackage{caption}
\usepackage{subcaption}
\makeatletter
\gdef\@copyrightpermission{
  \begin{minipage}{0.2\columnwidth}
   \href{https://creativecommons.org/licenses/by/4.0/}{\includegraphics[width=0.90\textwidth]{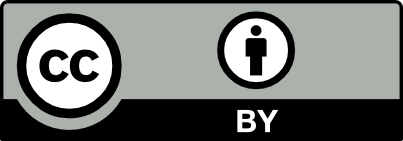}}
  \end{minipage}\hfill
  \begin{minipage}{0.8\columnwidth}
   \href{https://creativecommons.org/licenses/by/4.0/}{This work is licensed under a Creative Commons Attribution International 4.0 License.}
  \end{minipage}
  \vspace{5pt}
}
\makeatother
\copyrightyear{2025} 
\acmYear{2025} 
\setcopyright{rightsretained}
\acmConference[KDD '25]{Proceedings of the 31st ACM SIGKDD Conference on Knowledge Discovery and Data Mining V.1}{August 3--7, 2025}{Toronto, ON, Canada}
\acmBooktitle{Proceedings of the 31st ACM SIGKDD Conference on Knowledge Discovery and Data Mining V.1 (KDD '25), August 3--7, 2025, Toronto, ON, Canada}
\acmDOI{10.1145/3690624.3709324}
\acmISBN{979-8-4007-1245-6/25/08}
\begin{document}

\title{Understanding Oversmoothing in Diffusion-Based GNNs From the Perspective of Operator Semigroup Theory}

\author{Weichen Zhao}

\affiliation{%
  \institution{The School of Statistics and Data Science, LPMC \& KLMDASR, Nankai University}
  \city{Tianjin}
  \country{China}
}
\email{zhaoweichen@nankai.edu.cn}

\author{Chenguang Wang}
\affiliation{%
  \institution{The School of Data Science, \\ The Chinese University of Hong Kong, Shenzhen}
  \city{Shenzhen}
  \country{China}}
\email{chenguangwang@link.cuhk.edu.cn}

\author{Xinyan Wang}
\affiliation{%
  \institution{The Academy of Mathematics and Systems Science, Chinese Academy of Sciences}
  \city{Beijing}
  \country{China}
}
\email{wangxinyan21a@mails.ucas.ac.cn}

\author{Congying Han}
\authornote{Corresponding author}
\affiliation{%
 \institution{School of Mathematical Sciences, University of Chinese Academy of Sciences}
  \city{Beijing}
  \country{China}
  }
\email{hancy@ucas.ac.cn}

\author{Tiande Guo}
\affiliation{%
 \institution{School of Mathematical Sciences, University of Chinese Academy of Sciences}
  \city{Beijing}
  \country{China}}
\email{tdguo@ucas.ac.cn}

\author{Tianshu Yu}
\affiliation{%
  \institution{The School of Data Science, \\ The Chinese University of Hong Kong, Shenzhen}
  \city{Shenzhen}
  \country{China}}
\email{yutianshu@cuhk.edu.cn}
\renewcommand{\shortauthors}{Weichen Zhao et al.}



\begin{abstract}
This paper presents an analytical study of the oversmoothing issue in diffusion-based Graph Neural Networks (GNNs). Generalizing beyond extant approaches grounded in random walk analysis or particle systems, we approach this problem through operator semigroup theory.
        This theoretical framework allows us to rigorously prove that oversmoothing is intrinsically linked to the ergodicity of the diffusion operator. Relying on semigroup method, we can quantitatively analyze the dynamic of graph diffusion and give a specific mathematical form of the smoothing feature by ergodicity and invariant measure of operator, which improves previous works only show existence of oversmoothing. This finding further poses a general and mild ergodicity-breaking condition, encompassing the various specific solutions previously offered, thereby presenting a more universal and theoretically grounded approach to relieve oversmoothing in diffusion-based GNNs. Additionally, we offer a probabilistic interpretation of our theory, forging a link with prior works and broadening the theoretical horizon. 
        Our experimental results reveal that this ergodicity-breaking term effectively mitigates oversmoothing measured by Dirichlet energy, and simultaneously enhances performance in node classification tasks. 
\end{abstract}

%
%






\maketitle
\section{Introduction}\label{sec.1}
Graph Neural Networks (GNNs) have emerged as a powerful tool for learning graph-structured data, finding applications in various domains such as materials science~\cite{merchant2023scaling}, bioinformatics~\cite{zhang2023protein, Dual_KDD23}, and recommendation systems~\cite{qin2024learning}. In recent years, continuous GNNs~\cite{xhonneux2020continuous} are proposed to generalize previous graph neural networks with discrete dynamics to the continuous domain by Neural ODE~\cite{chen2018neural}.
Graph diffusion \cite{chamberlain2021grand,song2022robustness,choi2023gread} further extended the message-passing mechanism in classic GNNs under a partial differential equation (PDE) perspective. These works have made significant progress in terms of interpretability, stability, heterogeneous graphs, and beyond.

Despite these advances, a fundamental challenge of GNNs lies in the phenomenon of oversmoothing \cite{loss_exponentially,wu2024demystifying, improving_KDD23}, where repetitions of message passing may cause node representations to become indistinguishable and thus lose their discriminative power. For GNNs with discrete dynamics, several works \cite{jumping,pairnorm,simple_deep,rusch2022gradient} are proposed to relieve the oversmoothing issues. Additionally, recent works \cite{oscillator,thorpe2022grand++, wang2022acmp} have verified the existence of oversmoothing in GNNs with continuous dynamics and most of them address this issue by introducing additional terms in graph diffusion equations, such as source term \cite{thorpe2022grand++}, Allen-Cahn term \cite{wang2022acmp}, and reaction term \cite{choi2023gread}.
However, such extra terms to graph diffusion are often under specific physical scenarios without a generic and unified overview, resulting in case-specific solutions with narrow applicability.

In this paper, we propose a unified framework using operator semigroup theory to perform a principled analysis, then address this limitation. By viewing node features as solutions to the Cauchy problem associated with linear graph diffusion, we provide an in-depth understanding of the connection between ergodicity of operator and oversmoothing. 
In contrast to previous work studying oversmoothing through smooth measures such as Dirichlet energy~\cite{MaskeyPBK23,wu2024demystifying}, with operator semigroups we can not only confirm the existence of oversmoothing, but also make it clear that the speed of oversmoothing is exponential and dependent on the spectral gap of the operator. Further, the operator semigroup approach gives a specific form for the fixed point, which improves the work that merely know that the fixed point exists. We also discuss the nonlinear case through the weak ergodicity of the nonlinear graph diffusion operator. 
Building on this foundation, we propose a general and mild ergodicity-breaking condition, which accommodates specific solutions from previous research and further offers a more universal rule for designing terms to mitigate oversmoothing in diffusion-based graph neural networks.

Moreover, we supplement our theoretical contributions with a probabilistic interpretation by studying the Markov process in which the generator is a graph diffusion operator, thus establishing a comprehensive link with existing literature. Furthermore, we construct the killing process for graph diffusion which provides an intuitive probabilistic connection for the proposed ergodicity-breaking condition.

Our experimental results confirm the effectiveness of our theoretical results, demonstrating reduced oversmoothing, as evidenced by higher Dirichlet energy, and improved node classification performance. 

In summary, this paper makes several key contributions to the study of oversmoothing problem: (1) We introduce a comprehensive framework based on operator semigroup theory to analyze the oversmoothing issue in diffusion-based graph neural networks. From this framework, specific mathematical form of the smoothing feature and rate of convergence can be given by the quantity associated with graph diffusion operator; 
(2) Our work proposes an ergodicity-breaking condition that not only addresses the oversmoothing problem but also encompasses several specific extra terms identified in prior works, demonstrating its broad applicability;
(3) We provide a probabilistic interpretation of our method, thereby establishing a connection with previous theoretical analyses and enriching the overall understanding of diffusion-based graph neural networks dynamics;
(4) We substantiate our theoretical results through synthetic and real-world experiments.

\subsection{Related Work}
\textbf{Diffusion-based graph neural networks.} Treating GNNs as the discretization of the continuous dynamical system is a rapidly growing sub-field of graph representation learning \cite{chamberlain2021grand,BehmaneshKO23}. Since the message passing (MP) mechanism shows an intrinsic link to the diffusion process, several diffusion-based graph neural networks \cite{PDE-GCN, plaplacian, song2022robustness} are proposed and conducted using ODE solver \cite{chen2018neural}. GRAND \cite{chamberlain2021grand} parameterizes the underlying graph diffusion process to learn the node representations. BLEND \cite{Beltrami} considers graph as the discretization of a manifold and jointly conducts continuous feature learning and topology evolution based on the Beltrami flow.

\noindent\textbf{Oversmoothing.}
Oversmoothing refers to the effect that node features of graph neural networks tend to become more similar with the increase
of the network depth, constraining the model expressive power for many graph neural networks \cite{thorpe2022grand++, loss_exponentially}. Many previous graph neural networks aim at overcoming oversmoothing \cite{jumping, simple_deep, pairnorm}. GRAND++  \cite{thorpe2022grand++} relieves the oversmoothing problem by introducing the source term. Several GNNs quantitatively tackle the oversmoothing problem by analyzing Dirichlet energy \cite{oscillator, rusch2022gradient}. Based on the Allen-Cahn particle system with repulsive force, ACMP \cite{wang2022acmp} shows adaption for node classification tasks with high homophily difficulty. GREAD \cite{choi2023gread} introduces the reaction diffusion equation, encompassing different types of popular reaction equations, and empirically mitigating oversmoothing. FROND~\cite{Kang24} interprets the feature updating process on graphs as a non-Markovian random walk which converges to the stationary distribution at algebraic rate thus relieving oversmoothing. Our method provides a more general framework and a more theoretically grounded method to address the oversmoothing problem.
\begin{figure}
  \includegraphics[width=.48\textwidth]{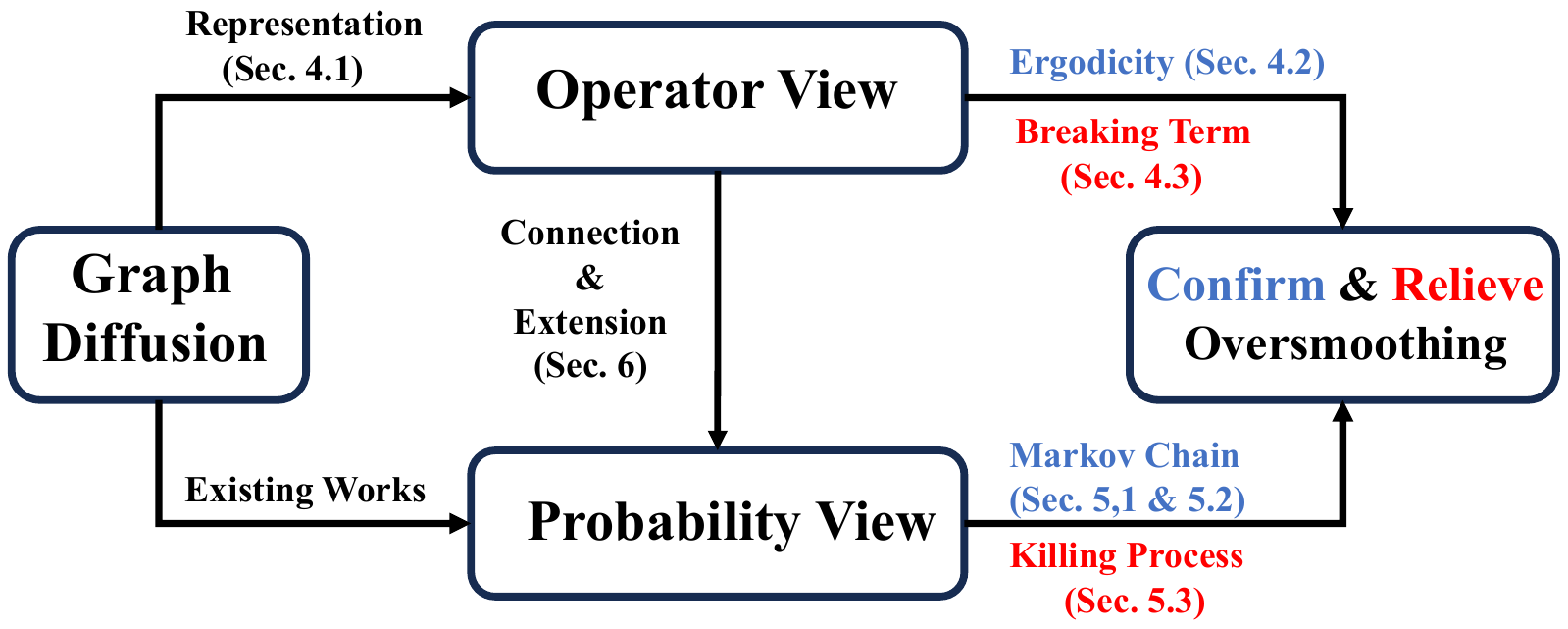}
  \caption{Paper Organization}
  \label{fig:organization}
\end{figure}

\subsection{Organization} 
In Sec. \ref{sec.2}, we provides a succinct introduction to the concepts of graph diffusion and oversmoothing. In Sec. \ref{sec.3}, we delve into the operator semigroups and generators of graph diffusion, which are pivotal mathematical constructs for modeling the dynamics of diffusion-based GNNs. Sec. \ref{sec.4} explores oversmoothing from the viewpoint of operator semigroups. In particular, Theorem \ref{thm:ergo_graph} in Sec. \ref{sec.ergodicity} elucidates the link between oversmoothing and the ergodic properties of the generator. Building on this foundation, Theorem \ref{theorem.extra term} in Sec. \ref{sec.4.3} proposes a novel approach to mitigate oversmoothing. 
Sec. \ref{sec.5} enhances the discussions in Sec. \ref{sec.4} by providing a complementary probabilistic perspective as the connection and generalization on previous works. Theorem \ref{thm5.3} validates that the strategy introduced in Theorem \ref{theorem.extra term} possesses a coherent probabilistic interpretation.
The organization of the main results of this work is outlined in Figure \ref{fig:organization}.

\section{Background}\label{sec.2}
        \textbf{Notations. } Let $ \mathcal{G}=(\mathcal{V},\mathcal{E}) $ be a graph, node be $ u\in\mathcal{V} $, edge be $ (u,v)\in\mathcal{E} $, $ W $ and $ D $ be adjacency matrix and degree matrix of graph $ \mathcal{G} $ respectively, $\mathcal{N}(u)$ be neighbors of node $ u\in\mathcal{V} $. Let $X
        \in \mathbb{R}^{n\times d}$ be the nodes features, where each row $ X(u)\in\mathbb{R}^d$ corresponds to the feature of node $u$.

       \noindent\textbf{Graph Neural Networks.} GNNs have emerged as a powerful framework for learning representations of graph-structured data. Message passing neural networks (MPNNs) \cite{gilmer2017neural}  
        incorperates two stages: aggregating stage and updating stage. The core idea behind MP is the propagation of information across the graph. Inspired by MPNNs, many graph neural networks typically operate by iteratively aggregating and updating node-level features based on the features of neighboring nodes. 
        GCNs~\cite{kipf2016semi} and GATs~\cite{velivckovic2017graph} are two prominent representatives in the field. GCNs leverage graph convolutional layers to capture and model relational information.
         GATs introduce the attention mechanism, allowing nodes to selectively attend to their neighbors during information propagation.

    \noindent\textbf{Graph Diffusion. }Neural Ordinary Diffusion Equations (Neural ODEs) \cite{chen2018neural} present an innovative approach in constructing continuous-time neural network, where the derivative of the hidden state is parametrized using the neural networks.
        \citet{chamberlain2021grand} extends this framework to construct continuous-time graph neural networks.
        Graph diffusion is the process through which features propagate across the nodes of a graph. More formally, the graph diffusion on a graph $\mathcal{G}$ is defined as follows:
        \begin{equation} \label{eq.diff_high dim}
            \frac{\partial H(t)}{\partial t}:= [A(H(t))-I]H(t),
        \end{equation}    	
    	\begin{equation}\label{eq.1}
        	\frac{\partial h(u,t)}{\partial t}:=\sum_{v \in \mathcal{V}} a(h(u,t),h(v,t))(h(v,t)-h(u,t)),
    	\end{equation}
        where $I$ is identity matrix, $A(H(t))=(a(h(u,t),h(v,t)))_{u,v \in \mathcal{V}}$ is the attention matrix~\cite{velivckovic2017graph}, where
        \begin{equation}\label{attention_matrix}
            a(h(u,t),h(v,t))=\frac{\exp{(b^{(t)}_{u,v}})}{\sum_{w \in \mathcal{N}(u)} \exp{(b^{(t)}_{u,w})}},\; (u,v) \in \mathcal{E}, 
        \end{equation}
        satisfying $\sum_{v\in \mathcal{V}} a(h(u,t),h(v,t))=1,\;\forall u$, where $\Phi: \mathbb{R}^d\times \mathbb{R}^d\rightarrow \mathbb{R}$ is \emph{attention function} to compute the weight between $u$ and $v$,
        $b^{(t)}_{u,v}:=\Phi(h(u,t), h(v,t))$. When $A(H(t))$ solely relies on initial features, implying $A(H(t))=A$, Eq.~\eqref{eq.1} is in the linear form:
        \begin{equation}\label{eq.graph diff linear}
            \frac{\partial h(u,t)}{\partial t}:=\sum_{v \in \mathcal{V}} a(u,v)(h(v,t)-h(u,t)).
        \end{equation}
        In this paper, we mainly concerned linear diffusion-based graph neural networks. Further discussion on nonlinear case is in section \ref{sec.ergodicity}. Moreover, by introducing a time discretization scheme, we can elucidate the forward propagation of the diffusion equation using the ODE solver. 
        It is worth emphasizing that many existing graph neural networks can be viewed as realization of the graph diffusion model by setting the feasible step and attention coefficients \cite{chamberlain2021grand}. 
        This linkage allows us to leverage the properties of ODEs to analyze a diverse range of graph neural network models.

   \textbf{}In graph neural networks, oversmoothing problem refers to the phenomenon where node features become excessively similar along 
    with the propagation process.
    In the context of graph diffusion, the oversmoothing problem can be defined as follows:
    \begin{definition}[The oversmoothing problem]
        There exists a constant vector $b\in\mathbb{R}^d$, such that $h(u,t) \rightarrow b$, as $t \rightarrow \infty$, for every $u \in \mathcal{V} $.
    \end{definition}
    We do not define oversmoothing here by a smoothness metric such as Dirichlet energy and emphasize exponential convergence because the necessity of exponential convergence as a condition for defining oversmoothing remains subject to debate. We opted not to adopt a definition that might spark controversy but instead expressed oversmoothing in mathematical forms based on the most fundamental intuitive understanding~\cite{thorpe2022grand++,nguyen2023revisiting,keriven2022not}. Indeed, exponential convergence can also be demonstrated from our proofs.

\section{Operator Semigroup and Generator of Graph Diffusion} \label{sec.3}
    In this section, we offer an extensive introduction to pertinent operator semigroup theory and associated theorems. These theoretical concepts are fundamental mathematical tools for the comprehensive analysis of the long time behavior of dynamical systems on graphs. 

    For $\mathcal{V} = \{1, 2, \ldots, n\}$, let 
     $f: \mathcal{V} \rightarrow \mathbb{R}$ defined on node set and $C(\mathcal{V})$ represent the set of all real-valued function on the node set. The operator is a mapping from function to function. 
     The \emph{operator semigroup} denoted as $\mathbf{S}=\{S 
     _t\}_{t\ge 0}$ comprises a family of operators with semigroup property: $S_r \circ S_t= S_t \circ S_r=S_{t+r}$, for $t, r \geq 0,$ and $S_0$ is an identity operator. 
    Given operator semigroup $S_t$, operator 
      \begin{equation}\label{def_generator}
          \mathcal{L}=\lim \limits_{r \to 0} \frac{1}{r} \left[S_r-\mathcal{I} \right]
      \end{equation}
      is called the \emph{infinitesimal generator} of $\mathbf{S}$, where $\mathcal{I}$ is the identity operator. $\mathcal{L}$ is uniquely determined given semigroup $S_t$ .The semigroup property of the operators $S_t, t \geq 0$ shows that for $t, r>0$,
    $
    \frac{1}{r}\left[S_{t+r}-S_t\right]=S_t\left(\frac{1}{r}\left[S_r-\mathcal{I}\right]\right)=\left(\frac{1}{r}\left[S_r-\mathcal{I}\right]\right) S_t .
    $
    Let $r \rightarrow 0$, 
    \begin{equation}\label{eq.Fokker}
    \frac{\partial S_t}{\partial t} = \mathcal{L}S_t = S_t\mathcal{L}.
    \end{equation}
    \zwc{If $S_t$ is a Markov semigroup, equation ~\eqref{eq.Fokker} is referred to as Kolmogorov's backward equation.}
    
    From the RHS of Eq.~\eqref{eq.graph diff linear}, we can derive an operator
    \begin{equation}\label{eq.2}
    \mathcal{A}f(u) := \sum_{v\in \mathcal{V}}a(u,v)(f(v)-f(u)),
    \end{equation}
    Then operator semigroup of graph diffusion induced by $\mathcal{A}$ can be well defined using Kolmogorov's equation,  denoted as $\mathbf{P}=\{P_t, t\ge 0\}$.  
        If $f$ represents a $d$-dimensional function, the operator $\mathcal{A}$ remains well-defined. It can be perceived as operating independently on each dimension. 
From the definition of graph diffusion, $\mathbf{P}$ can be shown to hold the following properties:   
    \begin{proposition}\label{prop.1}
        $\mathbf{P}=\{P_t, t\ge 0\}$ satisfies
    		\begin{itemize}
    			\item[(i)] $P_t 1=1$, where $1$ is a d-dimensional constant vector function, where each element is equal to 1.
    			\item[(ii)] If $f\ge 0$, then $P_t f\ge 0$.
    		\end{itemize}
	   \end{proposition}
    Property (i) is commonly known as the Markov property, meaning that the graph diffusion semigroup $P_t$ can be classified as a Markov semigroup. Furthermore, $\mathbf{P}$ has a clear probabilistic interpretation, as $P_t$ acts as a transition function for a Markov chain. In the context of graph diffusion, $P_t$ represents an $n \times n$ transition probability matrix, where the row sums equal 1 and the transition probabilities between non-adjacent nodes are zero. For a more comprehensively probabilistic interpretation, please refer to Sec. \ref{sec.5}. 
    

	
\section{Operator View for Oversmoothing}\label{sec.4}
    In this section, we illustrate that the graph diffusion equation~\eqref{eq.graph diff linear}, can be formulated as the Cauchy problem, and the solution of the graph diffusion equation has an operator semigroup format. We further discern that oversmoothing problem arises from the ergodicity of operators. In specific, the hidden layer features $h(u,t)$ converge to a fixed point as $t\rightarrow\infty$, for all nodes $u \in \mathcal{V}$. To address this issue, we introduce a condition aimed at breaking ergodicity and hence relieving the oversmoothing problem, offering a more universal rule for designing terms to relieve oversmoothing in diffusion-based GNNs.

	\subsection{Operator Representation of Feature}\label{sec.4.1}
   The dynamic of GRAND~\citep{chamberlain2021grand} is given by the graph diffusion equation. Given node input features $ X(u)=[X_{1}(u),X_{2}(u),\ldots,X_{d}(u)], u\in\mathcal{V}. $  The node features $ h(u,t) $ are the solution to the following Cauchy problem
	\begin{equation}\label{eq.4}
	\begin{cases}\frac{\partial h_i(u,t)}{\partial t}:=\mathcal{A}h_i(u,t), \\h_i(u,0)=f_i(u).\end{cases}\quad \forall u \in \mathcal{V},
	\end{equation}
    where $f_i(u)=X_{i}(u),i=1,\ldots,d$ are initial values. The Kolmogorov's equation (\ref{eq.Fokker}) provides us with the solution to the Cauchy problem, leading to the following theorem:
    
    \begin{theorem}\label{prop:cauchy_solu}
      Let $\mathbf{P}=\{P_t\}_{t\ge 0}$ be the graph diffusion operator semigroup, then the solution to graph diffusion in Cauchy problem form (\ref{eq.4}) is given by:
    \begin{equation}
    h_i(u,t)=P_tf_i(u).
    \end{equation}
	   \end{theorem}
    The result in one dimension can naturally extend to multiple dimensions:
        \begin{equation}\label{eq.op_repre}
            \begin{cases}
                h(u,t)=P_t f(u),\\ f(u)=X(u),          
            \end{cases}
            \quad \text{for } u \in \mathcal{V},
        \end{equation}
     Discussion in section \ref{sec.ergodicity} shows the advantage of this form becomes, as we can derive specific form of the smoothing feature by ergodicity of operator.
    
	\subsection{Ergodicity and Oversmoothing}\label{sec.ergodicity}
    

     We say that a measure $ \pi $ on measurable space $ (\mathcal{V},\mathscr{S}) $ is \emph{invariant} with respect to operator $ \mathcal{L} $ if for every bounded measurable function $ f $ satisfying
	\begin{equation}               
	\sum_{u\in\mathcal{V}}\mathcal{L}f(u)\pi(u)=0.       
	\end{equation} 
 
        We then define the ergodicity of generator $  \mathcal{L} $ as follows:
    	\begin{definition}[Ergodicity of $ \mathcal{L} $]
    		 The generator $  \mathcal{L} $ is said to be ergodic if every $ f  $ such that $ \mathcal{L}f = 0 $ is constant.
    	\end{definition}

     Next, we proceed to elaborate in the context of graph diffusion. It can be verified that the generator of graph diffusion $\mathcal{A}$ is ergodic, leading to the following theorem. 
    
	\begin{theorem}\label{thm:ergo_graph}
	    For a connected and non-bipartite graph $\mathcal{G}$, the operator 
		$ \mathcal{A} $ is ergodic.
	\end{theorem}

	A detailed proof is provided in Appendix \ref{appendix.4.3}. From this theorem, it can be seen that the ergodicity of operator is dependent on the topology property of the graph. Notice that $0\leq a(u,v)<1$, from Perron-Frobenius theorem (Theorem 8.4.4 of \cite{horn2012matrix}), the invariant measure $\mu$ of $\mathcal{A}$ exists and is unique. Then the ergodicity of operator $\mathcal{A}$ ensures 
     $P_t f(u)$ converges to the same constant for every $u \in \mathcal{V}$, as time $t \rightarrow \infty$. This finding is the cause of oversmoothing as different nodes share the same convergence. Under the same assumption in Theorem \ref{thm:ergo_graph}, above discussion can be summarized into the following theorem. Proof is provided in Appendix \ref{appendix.4.4}.
    

	\begin{theorem}\label{thm.graph_ergodicity_conv}
		  Let $P_t$ be the operator semigroup with infinitesimal generator $\mathcal{A}$, and $f(u)$ be the initial value mentioned in Eq.~\eqref{eq.op_repre}.
            then as $ t\rightarrow\infty $,
		\begin{equation}\label{eq.11}
			P_tf(u)\rightarrow\sum_{v\in \mathcal{V}}f(v) \mu(v)
		\end{equation}
  		where $ \mu $ is invariant for semigroup $ \mathbf{P}=\{P_t\}_{t\ge 0} $ with generator $\mathcal{A}$.
	\end{theorem}

 \begin{remark}
      From the proof of Theorem \ref{thm.graph_ergodicity_conv}, we not only demonstrate that oversmoothing arises in diffusion-based GNNs but also that the convergence rate is associated with the spectral gap of $\mathcal{A}$.
      Furthermore, the results of Theorem \ref{thm.graph_ergodicity_conv} elucidate the specific mathematical form of the fixed point of oversmoothing, namely, the initial node features averaged over the invariant measure $\sum_{v \in \mathcal{V}} f(v) \mu(v)$. This constitutes an advancement over prior work, which only established the exponential contraction of certain smooth measures.
 \end{remark}
 
 Since $ P_{t}f $ converges to the eigenspace corresponding to the eigenvalue $ \lambda_{0}=0 $, once the generator is not ergodic, the projection of function $ f\in\mathscr{L}^{2}(\mathcal{V},\mu) $ onto the eigenspace corresponding to the eigenvalue $ 0 $ is not constant. Therefore, $ \lim\limits_{t\rightarrow\infty}P_{t}f $ is not constant. That means non-oversmoothing. Using the notation from Eq.~\eqref{eq.op_repre} and combining with Theorem \ref{thm.graph_ergodicity_conv}, we can derive the following convergence:
     \begin{equation}
        h(u,t) \rightarrow \sum_{v \in \mathcal{V}} X_v \mu(v).
    \end{equation}
    This provides a fundamental mathematical explanation for the oversmoothing in diffusion-based GNNs. The underlying reasons can be attributed to the ergodicity of the operator. Consequently, to relieve oversmoothing, it is natural to break ergodicity of operator, as shown in the section \ref{sec.4.3}.
    
 \vspace{0.5em}
\noindent\textbf{Further Discussion on Nonlinear Case.} In addition to linear diffusion we further explore the more general case of over-smoothing in nonlinear diffusion models. We discuss the connection between weak ergodicity of operators and over-smoothing in nonlinear case.

The nonlinear diffusion-based GNNs is defined as follows
\begin{equation*} 
    \frac{\partial H(t)}{\partial t}:= [A(H(t))-I]H(t).
\end{equation*} 
After time discretization can be written as
$ H_{t+1} = A(H_{t}) H_{t}=A(H_{t})\cdots A(H_{0})H_{0}, \; t=0,1,2,\ldots $
where $H_0=X$. In this subsection we denote operator 
\begin{equation}\label{eq.discrete_P}
    P_t:=A(H_{t})\cdots A(H_{0}).
\end{equation} For all $t=0,1,2,\ldots$, $A(H_{t})$ is a stochastic matrix. Thus the family of operators $\{P_t\}, t=0,1,2,\ldots$ are the time inhomogeneous Markov operators, satisfies
\begin{equation}\label{eq.14}
    h_t(u)=P_tf(u)=\sum_{v\in\mathcal{V}}P_t(u,v)f(v),\quad u\in\mathcal{V}
\end{equation}
where $f(u):=x(u)=h_0(u)$. 

The research on the ergodicity of time inhomogeneous Markov operator is an important topic in probability~\cite{douc2004quantitative,saloff2009merging}. The following theorem shows the \emph{weak ergodicity} of the operator $P_t$ in \eqref{eq.discrete_P}. Proof is provided in Appendix \ref{sec.nonlinear}.

\begin{theorem}\label{thm_nonlinear}
    If The graph $\mathcal{G}$ is connected, non-bipartite and have self-loop on each node, assume the attention function $\Phi$ is continuous. Then Markov operator $P_t$ is weak ergodic, that is
    \begin{equation}
        \lim\limits_{t\to\infty}\|P_t(u,\cdot)-P_t(v,\cdot)\|_{\mathrm{TV}}=0,\quad \forall u,v\in\mathcal{V},
    \end{equation}
    where $\|\cdot\|_{\mathrm{TV}}$ is \emph{total variation distance}.
\end{theorem}


 From Theorem \ref{thm_nonlinear} we have 
 for all $u,v\in\mathcal{V}$,
\begin{align*}
    \lim\limits_{t\to\infty}h_t(u)&=\lim\limits_{t\to\infty}P_tf(u)=\lim\limits_{t\to\infty}\sum_{w\in\mathcal{V}}P_t(u,w)f(w)\\&=\lim\limits_{t\to\infty}\sum_{w\in\mathcal{V}}P_t(v,w)f(w)=\lim\limits_{t\to\infty}h_t(v).
\end{align*}
This shows that nonlinear diffusion-based GNNs still suffer from oversmoothing.
	\subsection{Ergodicity Breaking Term}\label{sec.4.3}
    We aim to relieve oversmoothing by breaking the ergodicity of the operator. This necessitates a modification of the operator $\mathcal{A}$ as presented in Eq.~\eqref{eq.2}. The central concept underpinning our approach involves introducing an additional term to the operator $\mathcal{A}$, yielding a new operator:
	\begin{equation}\label{eq.extra term}
	\tilde{\mathcal{A}}:=\mathcal{A}+\mathcal{C},
	\end{equation}
	where $ \mathcal{C} $ is an ergodicity-breaking term such that $ \mathcal{C}f\neq0 $ for all $ f $ satisfied $ \mathcal{A}f=0 $. In the simplest case, we can consider $ \mathcal{C}f(u):=c(u)f(u) $, where $ c(u) $ is a non-zero bounded function on $ \mathcal{V} $. 
	{Condition \eqref{eq.extra term} is highly permissive, making it applicable to a wide range of functions.}
	\begin{theorem}\label{theorem.extra term}
		The operator $\tilde{\mathcal{A}}$ is non-ergodic.
	\end{theorem}
    Proof for Theorem \ref{theorem.extra term} is furnished in Appendix \ref{appendix.4.5}. By substituting the operator $\mathcal{A}$ with the operator $\tilde{\mathcal{A}}$, we obtain a variation of the diffusion process (\ref{eq.graph diff linear}):
	$\frac{\partial h(u,t)}{\partial t}:=\mathcal{A}h(u,t)+\mathcal{C}h(u,t),$
     serving as a framework for designing GNNs to relieve oversmoothing. 

\section{Probability View for Oversmoothing}\label{sec.5}
	In this section, we aim to zoom into the oversmoothing of diffusion-based GNNs from a probabilistic perspective,
  which provides a probabilistic interpretation of the functional analysis findings in Sec.\ref{sec.4}. Firstly, we establish a connection between the node features $ h(u,t) $ and continuous time Markov chains on the graph. Secondly, we study the oversmoothing of diffusion-based GNNs through the ergodicity of Markov chains. Lastly, we show the inherent relationship between the ergodicity-breaking term and the killing process. 
 \subsection{Markov Chain on Graph with Generator $ \mathcal{A} $}
 

    Given a Markov process denoted as $\xi_t$ characterized by its state space $\mathscr{E}$ and a transition kernel denoted as $p(t,x,\mathrm{d}y)$, we can naturally establish a Markov semigroup as:
	$ P_{t}f(x):=\int_{E}f(y)p(t,x,\mathrm{d}y)=\mathbb{E}^{x}[f(\xi_{t})], $
	where $\mathbb{E}^{x}[\cdot]:=\mathbb{E}[\cdot|\xi_{0}=x]$ represents the expectation operator under the probability measure $p(\cdot,x,\mathrm{d}y)$. The generator of the Markov process $\xi_t$ can be obtained via Equation (\ref{eq.Fokker}). Conversely, in accordance with the Kolmogorov theorem (Theorem 2.11 of \cite{ge2011markov}), there exists a Markov process given its transition function and initial distribution. In this subsection, commencing from the generator $\mathcal{A}$ of graph diffusion, we formulate the corresponding process denoted as $\mathbf{X}$, which serves as a representation for the node features $H(t)$.
 
    Let $(\Omega,\mathscr{F},\mathbb{P})$ be a probability space and $\mathbf{Y}=\{\mathbf{Y}_n\}_{n\in\mathbb{N}}$ be a Markov chain with state space $(\mathcal{V},\mathscr{S})$ such that
	$ \mathbb{P}(\mathbf{Y}_{n+1}=v|\mathbf{Y}_{n}=u)=a(u,v). $
	The \emph{jumping times} are defined as: 
	$$ J_{n}:=\displaystyle\sum_{k=1}^{n}\frac{1}{1-a(\mathbf{Y}_{k-1},\mathbf{Y}_{k-1})}\zeta_{k},\; n=1,2,\ldots $$
	where $ \zeta_{k},\; k=1,2,\ldots,n $ is a sequence of independent exponentially distributed random variables of parameter $ 1 $.
	
	Next, we introduce the right-continuous process $\mathbf{X}: [0,\infty)\times\Omega\rightarrow X$ defined as:
	\begin{equation}\label{eq.3}
	\mathbf{X}_{t}:=\mathbf{Y}_{n},\;\text{if}\;t\in[J_{n},J_{n+1}).
	\end{equation}
     which is generated by $\mathcal{A}$ with the transition function denoted as ${p(t,u,v)=e^{t\mathcal{A}}\mathbf{1}_{v}(u)}{t\ge 0}$, where $\mathbf{1}_{v}(u)$ represents the indicator function.
 
    {In Sec.\ref{sec.4.1}, we have represented the node features $H(t)$ of diffusion-based GNNs in the form of an operator semigroup $\mathbf{P}$. The following theorem demonstrates that the node features can be further expressed as conditional expectations of $\mathbf{X}$.}
    \begin{theorem}\label{thm5.1}
		Let $\mathbf{X}$ defined in (\ref{eq.3}) be a continuous time Markov chain with generator $\mathcal{A}$, then 
		\begin{equation}
		h_{i}(u,t)=\mathbb{E}^{u}[f_i(\mathbf{X}_t)]
		\end{equation}
		is the solution of Cauchy problem (\ref{eq.4}).

	\end{theorem}
   
  Theorem \ref{thm5.1} is a variation of the well-known Feynman-Kac formula in stochastic analysis, which establishes a connection between a class of partial differential equations and stochastic differential equations \cite{freidlin1985functional}. Based on the connection between the node feature $H(t)$ and the Markov chain $\mathbf{X}$, we can analyze the dynamics of diffusion-based GNNs by studying the properties of $\mathbf{X}$.
    
    


\subsection{Limiting Distribution of Node Features}\label{sec.5.2}
Studying the transition function $P(t)$ of continuous-time Markov chains with a finite state space is an important research aspect in the field of Markov chains. The oversmoothing problem pertains to the asymptotic convergence property of node features as the model depth increases. In this subsection, we aim to analyze the dynamics of the node feature $h(u,t)$ by examining the transition function $P(t)$.
	\begin{theorem}\label{thm5.2}
		If $ \mathcal{G} $ is a finite connected graph, then for continuous time Markov chain $\mathbf{X}$ with generator $\mathcal{A}$, there exists a unique limiting probability measure $ \mu $ over $ \mathcal{V} $ such that
		\begin{equation}\label{eq: limit mu}
		      p(t,u,v)\rightarrow \mu(v),\quad t\rightarrow\infty. 
		\end{equation}
	\end{theorem}
 The limit of probability measure $\mu$ in Eq.\eqref{eq: limit mu} is indeed an invariant measure of the generator $\mathcal{A}$.
 Due to its uniqueness, Theorem \ref{thm5.2} shares the same measure $\mu$ as defined in Eq.\eqref{eq.11}.
 Theorem \ref{thm5.2} further reflects the limiting states of the transition semigroup $\mathbf{P}=\{P_t, t\ge 0\}$ corresponding to the generator $\mathcal{A}$. Combining with Theorem \ref{thm5.1}, we have
	\begin{align*}
		h_{i}(u,t)&=\mathbb{E}^{u}[f_i(\mathbf{X}_t)]=\sum_{v \in \mathcal{V}}f_{i}(v)p(t,u,v)\rightarrow \sum_{v \in \mathcal{V}}X_{i}(v)\mu(v) \;\;
	\end{align*}
 as $t\rightarrow\infty,$ showing that as the Markov chain transitions, each node feature tends to the average over the entire state space $\mathcal{V}$ of the input $X$. The uniqueness of such a limiting distribution stated in Theorem \ref{thm5.2} indicates that each node feature will converge to the same value, providing the probabilistic interpretation to the oversmoothing issue in diffusion-based GNNs.
 
\subsection{Killing Process}\label{sec.5.3}
In Sec. \ref{sec.4.3}, we introduced the ergodicity breaking term $\mathcal{C}$ to mitigate the oversmoothing issue. 
In this subsection, we proceed with the assumption that $\mathcal{C}f(u):=c(u)f(u)$, where $c(u)\leq 0$.  We shall delve into the study of the stochastic process with the generator $\tilde{\mathcal{A}}$ and provide a probabilistic interpretation of the ergodicity-breaking term.

We consider the killing process, which has a state space denoted as $ \tilde{\mathcal{V}}:=\mathcal{V}\cup{\partial} $, derived from $ \mathbf{X} $ as defined in Eq.(\ref{eq.3}):
\begin{equation}\label{eq.}
\tilde{\mathbf{X}}_{t}:= \begin{cases}
\mathbf{X}_{t},\quad t<\tau\\
\partial,\quad t\ge\tau
\end{cases}
\end{equation}
Here, $ \tau $ is a stopping time, and $ \partial $ signifies the \emph{dead state}. We define the function $c(u)$ as:
$
c(u):=-\lim\limits_{t\rightarrow 0^{+}}\frac{1}{t}\mathbb{P}(\tau\leq t|\mathbf{X}_{0}=u)$. The probability $ \mathbb{P}(\tau\leq t|\mathbf{X}_{0}=u) $ represents the likelihood of being terminated, starting from state $ u $ within the time interval $ (0,t] $. Consequently, $ c(u) $ can be comprehended as the negative of the killing rate. The following theorem asserts that the Markov process corresponding to the generator $ \tilde{\mathcal{A}}=\mathcal{A}+c(u) $ indeed matches the killing process $ \tilde{\mathbf{X}} $ defined in Eq. \eqref{eq.}:
\begin{theorem}\label{thm5.3}
The killing process $ \tilde{\mathbf{X}} $ is a Markov process with a generator $\tilde{\mathcal{A}}=\mathcal{A}+c(u) $.
\end{theorem}

     \begin{figure*}[htp]
    \centering
    \subcaptionbox{Oversmoothing on Wisconsin Dataset\label{os-wisconsin}}
    [.32\textwidth]{\includegraphics[width=.3\textwidth]{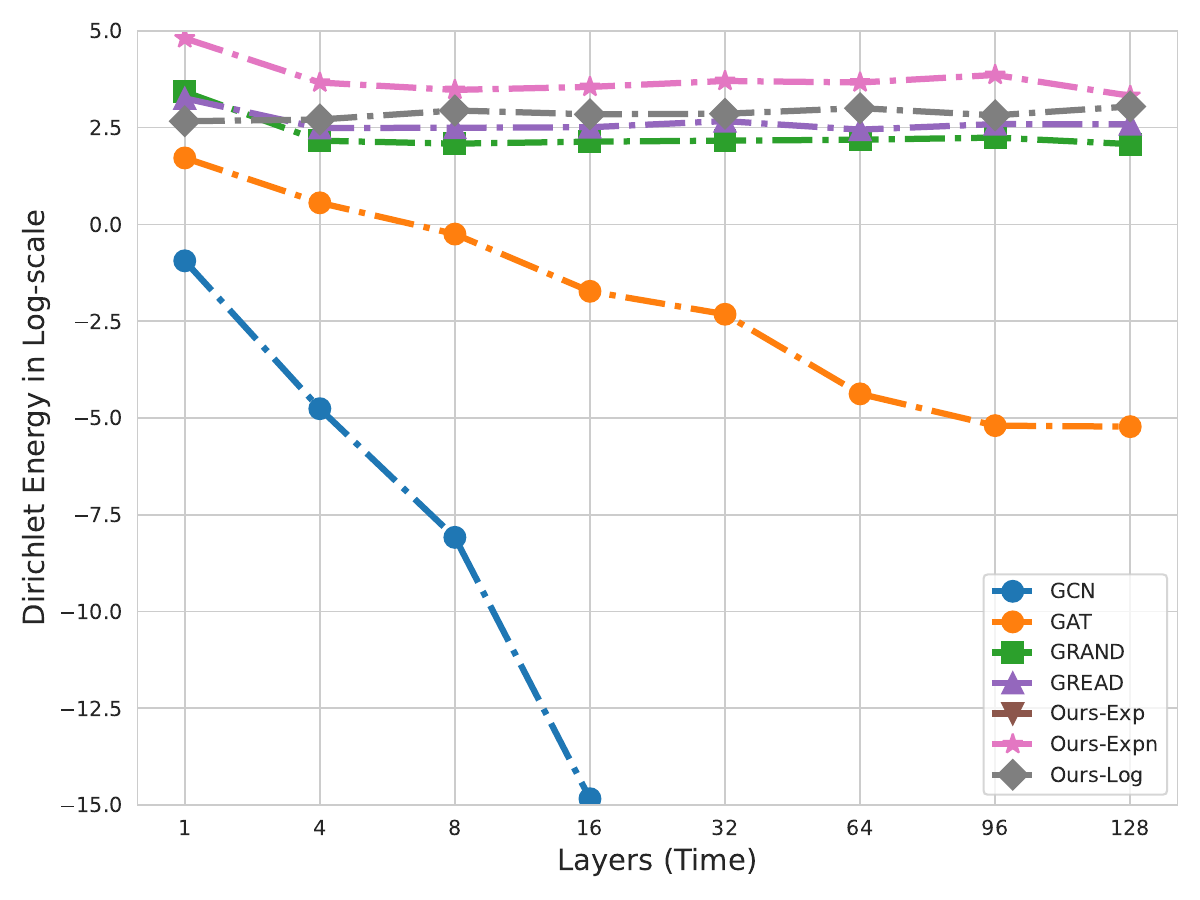}}
    \subcaptionbox{Effects of Homophily on Oversmoothing\label{os-homo}}
    [.32\textwidth]{\includegraphics[width=.3\textwidth]{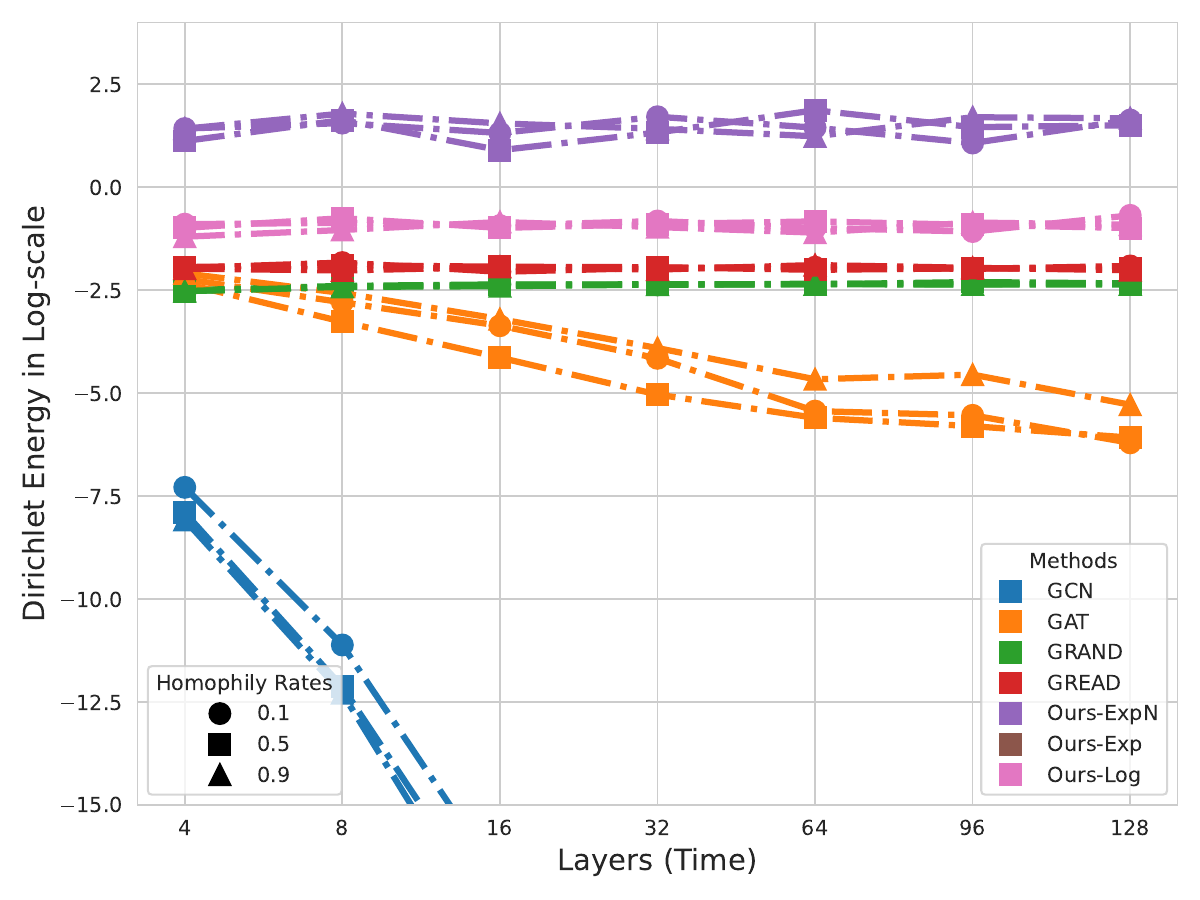}}
    \subcaptionbox{Performance as Layer Increases\label{os-test-prerformance}}
    [.32\textwidth]{\includegraphics[width=.3\textwidth]{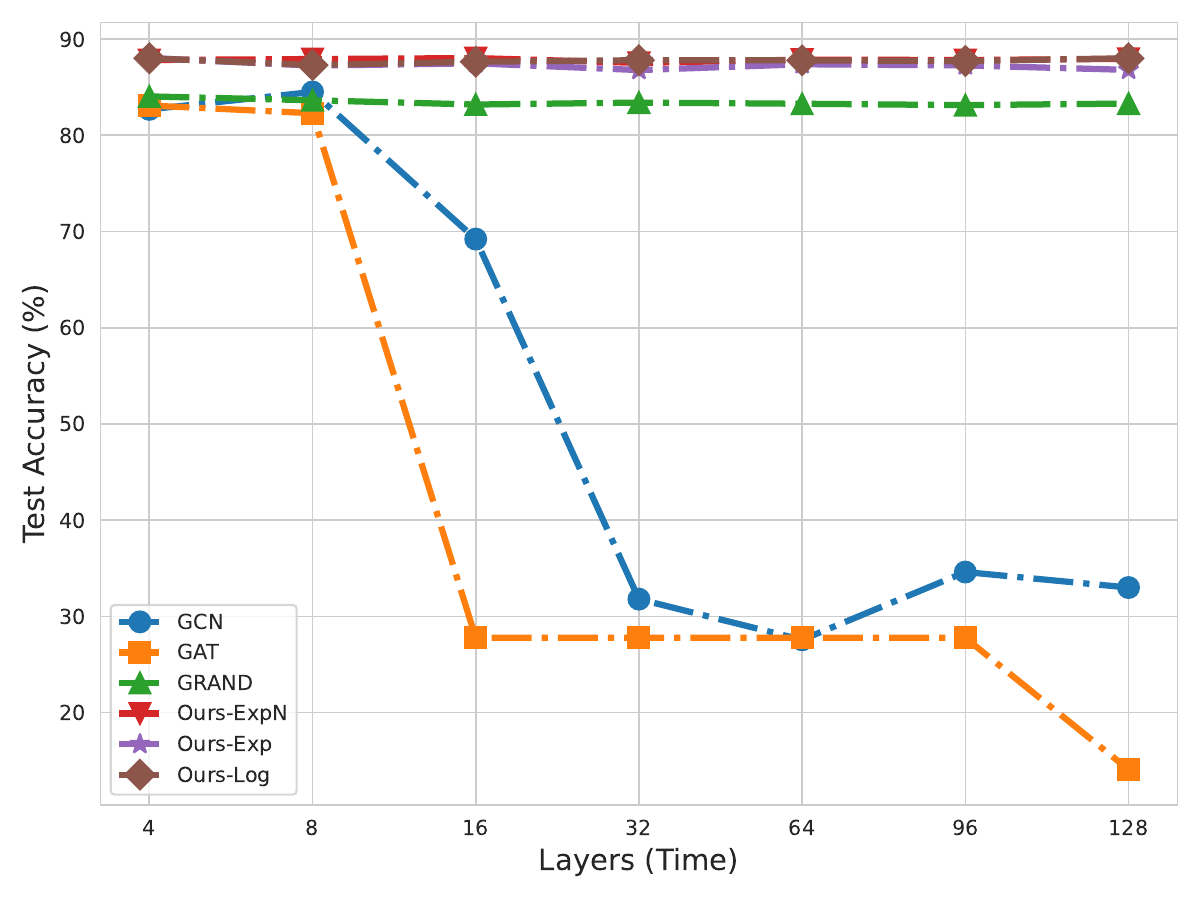}}
    
    \caption{Demonstration on Oversmoothing and Performance Metrics.
    The figure comprises three panels: \ref{os-wisconsin} showcases the increase in oversmoothing of various GNN models with increasing depth on the Wisconsin dataset, with models like GCN and GAT experiencing significant effects; \ref{os-homo} illustrates the impact of synthetic Cora dataset with varying homophily levels on oversmoothing; and \ref{os-test-prerformance} displays the corresponding decrease in test accuracy on Cora dataset as network depth increases, highlighting performance declines in models such as GCN and GAT at greater depths due to the oversmoothing effects noted in the other panels.
    }
    \label{fig:result os}
    \end{figure*}
 
    In reference to Section \ref{sec.5.2}, it has been established that the Markov process denoted as $ \mathbf{X} $, characterized by the generator $ \mathcal{A} $, will ultimately attain ergodicity across the entire state space. The probability measure $\mu(u)$, representing the likelihood of being in state $ u\in\mathcal{V} $, remains unchanged, thereby causing different node features to converge towards a common value, which results in the oversmoothing phenomenon. However, the killing process $ \tilde{\mathbf{X}} $ follows a distinct trajectory, as it transitions into a non-operational state, known as the ``dead state," after an unpredictable duration $\tau$. This behavior prevents it from achieving ergodicity throughout the state space and impedes the convergence of node features. It is worth noting that a similar concept is employed in the gating mechanisms frequently utilized in the field of deep learning. In fact, \citet{rusch2022gradient} has applied this gating mechanism to effectively address the oversmoothing problem.



\section{Connections and Extensions to Existing Works}\label{sec.connect and ext}
In this section, we explore the relationships between this work and the existing literature, highlighting both the continuities and the novelties introduced.
\subsection{Confirmation of Oversmoothing}
In the context of discrete Graph Neural Networks (GNNs), the process of feature propagation can be intuitively likened to a random walk occurring on the graph structure. \citet{zhao2022comprehensive} conducted an extensive investigation into the oversmoothing problem prevalent in classical GNN models like GCN and GAT. This analysis was conducted using discrete-time Markov chains on graphs. Additionally, \citet{thorpe2022grand++} delved into the oversmoothing issue within the GRAND framework, employing discrete-time random walks for their study. 

The probabilistic interpretation presented in section \ref{sec.5} can be seen as a broader extension of the insights derived from the discrete setting. Our research endeavors further delve into continuous-time Markov processes, with a particular emphasis on exploring the intricate relationships between stochastic processes, operator semigroups, and graph diffusion equations.

\subsection{Relief of Oversmoothing} 
Graph diffusion with ergodicity breaking term $ \mathcal{C} $ proposed in Sec. \ref{sec.4.3} is an explanatory framework for many existing diffusion-based GNNs that have been proven effective in addressing the oversmoothing problem. In previous attempts to mitigate the oversmoothing issue in diffusion-based GNNs, researchers have proposed several approaches. \citet{chamberlain2021grand} adds a source term in the graph diffusion equation, while \citet{wang2022acmp} alleviates the issue of oversmoothing by introducing an Allen-Cahn term. Based on the reaction-diffusion equations, \citet{choi2023gread} introduces the reaction term $r(H(t))$, encompassing the forms of the previous two approaches. In comparison to the aforementioned methods \cite{thorpe2022grand++, choi2023gread,wang2022acmp}, our framework offers a more rigorous theoretical foundation and provides the essential condition to deal with oversmoothing, that is, breaking ergodicity. Then the three methods mentioned earlier, which address the oversmoothing issue by incorporating additional terms, can be viewed as specific cases of our approach. 

\subsection{Technical Improvements}

The operator semigroup method serves as a robust tool for analyzing both linear and nonlinear dynamic systems. In this work, we do not endeavor to prove the exponential contraction of a smooth measure~\cite{nguyen2023revisiting,wu2024demystifying}. Instead, utilizing the semigroup method, we conduct a thorough analysis of the dynamics of the graph diffusion equation across varying initial conditions. By employing definitions of ergodicity and invariant measures associated with operators, we establish a specific mathematical expression for the smoothing feature, advancing beyond previous studies that merely confirmed the existence of oversmoothing. Additionally, the operator semigroup framework facilitates our investigation into the quantitative bounds of convergence rates through functional inequalities~\cite{bakry2014analysis} in future work. Furthermore, our approach fosters a deep connection between probability theory and functional analysis via operator semigroups and generators, offering a unified perspective that encompasses many earlier works framed within the context of Markov processes~\cite{thorpe2022grand++}.
\zwc{Graph signal processing (GSP) emerged as one of the pioneering frameworks for elucidating the phenomenon of oversmoothing~\cite{li2018deeper}. Within this framework, the graph convolution operator functions analogously to a low-pass filter, resulting in smoothing of node features. GSP is grounded in spectral graph theory, which emphasizes operators associated with the graph's structure such as the Laplacian. This foundation aligns with operator semigroups, which focus on operators that act on node features. While, semigroup methods show more power in continuous models.}

\section{Experimental Verification}
    This section presents empirical experiments designed to validate the theoretical analysis of the oversmoothing problem in graph diffusion. Our objectives are twofold: \textbf{(1)} Verification of oversmoothing: We empirically demonstrate the oversmoothing phenomenon in certain graph diffusion structures as theorized, alongside testing the proposed method aimed at mitigating this issue in scenarios characterized by differing levels of homophily and test performance; \textbf{(2)} Comparative Performance Analysis: The proposed method is evaluated against standard baselines in node classification benchmarks to illustrate its effectiveness in improving performance.

    \subsection{Specification on Ergodicity Breaking Terms}\label{sec: ebo}
    Following Theorem \ref{theorem.extra term}, we propose the following ergodicity breaking terms $\mathcal{C}$ w.r.t. the elementary function based on the adjacency matrix $A$: 
    (1) Exponential: $e^{A}=\sum_{n=0}^{\infty} \frac{1}{n !} A^{n} $;
        (2) Negative exponential: $e^{-A}=\sum_{n=0}^{\infty} \frac{(-1)^n}{n !} A^{n} $;
        (3) Logarithm: $\log{(I+A)}=\sum_{n=1}^{\infty} \frac{(-1)^{n-1}}{n} A^{n} $;
    In the implementation level, we truncate the series by $N\in\mathbb{N}^+$, resulting breaking terms with different orders.

\begin{table*}[!ht]
    \centering
        \caption{Node-classification results. Top three models are coloured by \textcolor{red}{First}, \textcolor{blue}{Second}, \textcolor{orange}{Third}}
    \label{tab: simple results}
    \resizebox{\textwidth}{!}{%
    \begin{tabular}{l ccccccccc}
    \toprule 
         &
         \textbf{Texas} &  
         \textbf{Wisconsin} & 
         \textbf{Cornell} &
         \textbf{Film} &
         \textbf{Squirrel} &
         \textbf{Chameleon} &
         \textbf{Citeseer} & 
         \textbf{Pubmed} & 
         \textbf{Cora} \\
         
         Hom level &
         \textbf{0.11} &
         \textbf{0.21} & 
         \textbf{0.30} &
         \textbf{0.22} & 
         \textbf{0.22} & 
         \textbf{0.23} &
         \textbf{0.74} &
         \textbf{0.80} &
         \textbf{0.81} \\ 
         
         
         
         \midrule
         
         
         $\text{PairNorm}$ &
         $60.27 \pm 4.34$ &
         $48.43 \pm 6.14$ &
         $58.92 \pm 3.15$ &
         $27.40 \pm 1.24$ & 
         $50.44 \pm 2.04$ & 
         $62.74 \pm 2.82$ &
         $73.59 \pm 1.47$ &
         $87.53 \pm 0.44$ &
         $85.79 \pm 1.01$ \\ 
         
         $\text{GraphSAGE}$  &
         $82.43 \pm 6.14$ &
         $81.18 \pm 5.56$ &
         $75.95 \pm 5.01$ &
         $34.23 \pm 0.99$ & 
         $41.61 \pm 0.74$ & 
         $58.73 \pm 1.68$ &
         $76.04 \pm 1.30$ &
         $88.45 \pm 0.50$ &
         $86.90 \pm 1.04$\\
         
         $\text{GCN}$ &
         $55.14 \pm 5.16$ &
         $51.76 \pm 3.06$ &
         $60.54 \pm 5.30$ &
         $27.32 \pm 1.10$ & 
         $53.43 \pm 2.01$ &
         $64.82 \pm 2.24$ &
         $76.50 \pm 1.36$ &
         $88.42 \pm 0.50$ &
         $86.98 \pm 1.27$ \\ 
         
         $\text{GAT}$ &
         $52.16 \pm 6.63$ &
         $49.41 \pm 4.09$ &
         $61.89 \pm 5.05$ &
         $27.44 \pm 0.89$ & 
         $40.72 \pm 1.55$ &
         $60.26 \pm 2.50$ &
         $76.55 \pm 1.23$ &
         $87.30 \pm 1.10$ & 
         $86.33 \pm 0.48$ \\ 
         

%
         $\text{CGNN}$ &
         $71.35 \pm 4.05$ &
         $74.31 \pm 7.26$ &	
         $66.22 \pm 7.69$ &	
         $35.95 \pm 0.86$ &	
         $29.24 \pm 1.09$ &	
         $46.89 \pm 1.66$ &	
         $\textcolor{black}{76.91 \pm 1.81}$ &	
         $87.70 \pm 0.49$ &	
         $87.10 \pm 1.35$ \\
        
        $\text{GRAND}$ &
        $75.68 \pm 7.25$ &
        $79.41 \pm 3.64$ &
        $74.59 \pm 4.04$ &
        $35.62 \pm 1.01$ &
        $40.05 \pm 1.50$ &
        $54.67 \pm 2.54$ &
        $76.46 \pm 1.77$ &
        $89.02 \pm 0.51$ &
        $87.36 \pm 0.96$ \\
         

        GRAFF & 
         $ 84.32\pm 5.51 $ &
         $ \textcolor{black}{85.68 \pm 3.40} $ &
         $ 71.08 \pm 3.64 $ &
         $ 35.55 \pm 1.20 $ &
         $ \textcolor{blue}{57.60 \pm 1.42} $ & 
         $ \textcolor{black}{68.75 \pm 1.83} $ & 
         $ {75.27 \pm 1.45} $ &
         $ {89.60 \pm 0.40} $ &
         $ {87.66 \pm 0.97}  $\\



        GREAD & 
         $ \textcolor{orange}{84.59 \pm 4.53} $ &
         $ 85.29 \pm 4.49 $ & 
         $ {73.78\pm 4.53} $ & 
         $ \textcolor{blue}{37.59 \pm 1.06} $ &
         $ \textcolor{red}{58.62 \pm 1.08} $ & 

        $ \textcolor{blue}{70.00 \pm 1.70} $ & 
         $  \textcolor{blue}{77.09 \pm 1.78} $ &
         $ \textcolor{orange}{90.01 \pm 0.42}$ &
         $ \textcolor{blue}{88.16 \pm 0.81}$\\
         \midrule

        Ours-Expn & 
         $ \textcolor{red}{88.65 \pm 2.64} $ &
         $ \textcolor{red}{88.82 \pm 3.62} $ & 
         $ \textcolor{orange}{76.48 \pm 2.71} $ & 
         $ \textcolor{red}{37.79 \pm 0.86} $ &
         $ \textcolor{black}{50.83 \pm 1.69} $ & 
        $ \textcolor{orange}{69.74 \pm 1.26} $ & 
         $ \textcolor{red}{77.48 \pm 1.26} $ &
         $ \textcolor{red}{90.08 \pm 0.49}$ &
         $ \textcolor{red}{88.31 \pm 0.85}$\\
        
        Ours-Exp & 
         $ \textcolor{black}{84.06 \pm 4.59} $ &
         $ \textcolor{blue}{87.64 \pm 3.51} $ & 
         $ \textcolor{blue}{77.57 \pm 3.83} $ & 
         $ \textcolor{black}{37.37 \pm 1.17} $ &
         $ \textcolor{black}{50.03 \pm 1.90} $ & 
        $ \textcolor{black}{69.41 \pm 1.15} $ & 
         $ \textcolor{black}{76.75 \pm 1.76} $ &
         $ \textcolor{black}{89.92 \pm 0.37}$ &
         $ \textcolor{black}{87.97 \pm 1.09}$\\

        Ours-Log & 
         $ \textcolor{blue}{87.03 \pm 5.24} $ &
         $ \textcolor{orange}{86.67 \pm 2.60} $ & 
         $ \textcolor{red}{78.11 \pm 3.90} $ & 
         $ \textcolor{orange}{37.55 \pm 0.80} $ &
         $ \textcolor{orange}{54.75 \pm 1.57} $ & 
        $ \textcolor{red}{70.75 \pm 1.39} $ & 
         $ \textcolor{orange}{77.07 \pm 1.60} $ &
         $ \textcolor{blue}{90.02 \pm 0.38}$ &
         $ \textcolor{orange}{88.07 \pm 1.24}$\\

         \bottomrule
         \bottomrule
    \end{tabular}
    }
\end{table*}

\begin{table}[h]
\centering
\caption{Further node-classification results. The best results are highlighted in bold.}
\label{tab: res2}
\resizebox{.45\textwidth}{!}{%
\begin{tabular}{l c c c c}
\toprule
\cmidrule{2-5}
 & Computers & Photo & CoauthorCS & ogbn-arxiv \\
\midrule
GRAND & $69.12\pm 0.41$ & $84.64\pm 0.25$ & $92.89\pm 0.30$ & $91.27\pm 0.17$ \\
GRAND++ & $68.66\pm 0.63$ & $84.99\pm 0.52$ & $92.89\pm 0.23$ & $91.51\pm 0.23$ \\
ACMP & $69.87\pm 0.33$ & $85.30\pm 0.66$ & $92.97\pm 0.41$ & $91.23\pm 0.36$ \\
GREAD & $69.89\pm 0.45$ & $85.81\pm 0.63$ & $92.93\pm 0.13$ & $90.93\pm 0.31$ \\
Ours-Exp & $70.35\pm 0.24$ & $85.49\pm 0.98$ & $93.71\pm 0.23$ & $91.86\pm 0.20$ \\
Ours-Expn & $\mathbf{70.85\pm 0.24}$ & $\mathbf{86.16\pm 0.54}$ & $\mathbf{94.03\pm 0.31}$ & $91.59\pm 0.32$ \\
Ours-Log & $70.68\pm 0.69$ & $86.02\pm 0.31$ & $93.75\pm 0.18$ & $\mathbf{92.04\pm 0.30}$ \\
\bottomrule
\end{tabular}
}
\vspace{-3mm}
\end{table}

    \subsection{Oversmoothing Validation}    

    We first empirically validate the oversmoothing phenomenon as posited in our theoretical analysis. Dirichlet energy \citep{MaskeyPBK23} is utilized as the primary metric for this purpose, defining as:
    $
     E\left(H, A\right)=\frac{1}{N} \sum_{u \in \mathcal{V}} \sum_{v \in \mathcal{N}(u)} a(u,v)\left\|\frac{h_{u}}{\sqrt{d_{u}}}-\frac{h_{v}}{\sqrt{d_{v}}}\right\|^{2},
    $
    where $d_u$ is the degree of node $u$.
    We adopt the synthetic Cora \cite{zhu2020beyond} dataset for our experiments. 
   
    To contextualize our findings, we compare the following models: GCN \cite{kipf2016semi}, GAT \cite{velivckovic2017graph}, GRAND \cite{chamberlain2021grand}, GREAD \cite{choi2023gread} and ergodicity-breaking terms in Sec. \ref{sec: ebo}, denoting as Ours-Exp, Ours-ExpN, and Ours-Log. Detailed configurations are in Appendix \ref{appendix. os}.

    The experimental findings delve into several critical aspects of oversmoothing issues that have been less emphasized in prior research:\\
    \noindent\textbf{Question: What's the impact of varying levels of homophily on oversmoothing?} \\
    \noindent\textit{Answer:}  
    Figure \ref{os-homo} illustrates the evolution of Dirichlet energy as the number of layers increases across different levels of homophily on synthetic Cora dataset. The results indicate that certain models exhibit greater sensitivity to changes in homophily than others. For instance, models such as GAT and GCN demonstrate a pronounced responsiveness to variations in homophily levels, suggesting that their performance in mitigating oversmoothing is considerably affected by these changes. In contrast, specific models within the "Ours" series, characterized by closely aligned lines, exhibit enhanced stability against oversmoothing across varying levels of homophily. This observation underscores the notion that while some GNN models maintain robustness against fluctuations in homophily, others may necessitate careful consideration of the homophily characteristics inherent in the dataset to achieve optimal performance.

    \noindent\textbf{Question: How does the performance change with increasing network depth?} \\
    \noindent\textit{Answer:} Figure \ref{os-test-prerformance} displays the test accuracy of different GNN models as a function of the number of layers on Cora dataset. The performance trends indicate a distinct variance among the models. For GCN and GAT, there is a marked decrease in performance as the number of layers increases, indicating a susceptibility to performance degradation due to oversmoothing. This effect is particularly pronounced beyond 32 layers where the test accuracy sharply declines.
    In contrast, for graph diffusion based models, performance remains relatively stable and high across all layer depths, suggesting these models are resilient to the deep layer effects, including potential oversmoothing issues.
    What's more, the "Ours" series methods achieve better performance than GRAND across the increasing layers, highlighting their effectiveness in handling deeper architectures.

    \subsection{Node Classification}\label{ncexp} 
    \noindent\textbf{Datasets.} We evaluate the performance of our proposed method in comparison with existing GNN architectures. The focus is on both heterophilic and homophilic graph datasets to showcase the model's versatility and robustness across diverse real-world scenarios. For heterophilic datasets, we utilize six datasets known for their low homophily ratios as identified in \cite{pei2020geom}, including Chameleon and Squirrel \cite{rozemberczki2021multi}, Film \cite{tang2009social}, Texas, Wisconsin, and Cornell datasets from the WebKB collection. For     homophilic datasets, we employ Cora \cite{mccallum2000automating}, CiteSeer \cite{sen2008collective} and PubMed \cite{yang2016revisiting}.
    For data splits, we adopt the methodology from \cite{pei2020geom}, ensuring consistency and comparability in our evaluations. The performance is gauged in terms of accuracy, with both mean and standard deviation reported. Each experiment is conducted over 10 fixed train/validation/test splits to ensure the reliability and reproducibility of results.
    In addition, following the settings in GRAND \cite{chamberlain2021grand}, additional datasets include the coauthor graph CoauthorCS \cite{mcauley2015image}, the Amazon co-purchasing graphs Computer and Photo \cite{mcauley2015image}, and the OGB Arxiv dataset \cite{hu2020open}.
    
    \noindent\textbf{Baselines.} Our model is benchmarked against a comprehensive set of GNN architectures, encompassing both traditional models: PairNorm \cite{pairnorm}, GCN \cite{kipf2016semi}, GAT \cite{velivckovic2017graph} and GraphSage \cite{hamilton2017inductive}, and recent ODE-based models: Continuous Graph Neural Networks (CGNN) \cite{xhonneux2020continuous},  GRAND \cite{chamberlain2021grand}, 
    GRAFF \cite{di2022graph} and GREAD \cite{choi2023gread}.

    \noindent\textbf{Hardwares.} All experiments reported in this work were conducted within a uniform computational environment to ensure reproducibility and consistency of results. The specifications of the environment are as follows:
    Operating System: Ubuntu 18.04 LTS;
    Programming Language: Python 3.10.4;
    Deep Learning Framework: PyTorch 2.0.1;
    Graph Neural Network Library: PyTorch Geometric 2.4.0;
    Differential Equation Solver: TorchDiffEq 0.2.3;
    GPU Computing: CUDA 11.7;
    Processor: AMD EPYC 7542 32-Core Processor;
    Graphics Card: NVIDIA RTX 3090.

    \noindent\textbf{Our Method} We use the ergodicity breaking terms in Sec. \ref{sec: ebo} with specific truncated orders $N$ and different ODE block, detailed optimal configurations are in Appendix \ref{appendix. nct}.

    \noindent\textbf{Results.} The data presented in Table \ref{tab: simple results} illuminates the comparative efficacy of diverse graph neural network models when applied to datasets characterized by varying degrees of homophily. Notably, the diffusion-based models, augmented with our proposed ergodicity breaking terms, demonstrate remarkable robustness and adaptability, emerging as the top-performing models on 8 out of 9 datasets. Furthermore, they consistently secure the top-three performance overall. Particularly noteworthy is their exceptional performance on datasets with low homophily levels, such as Wisconsin and Cornell, as well as on the Pubmed dataset, which exhibits a high level of homophily. For the additional results presented in Table \ref{tab: res2}, 'Ours-Expn' achieves the highest accuracy on the Computers, Photo, and CoauthorCS datasets, while 'Ours-Log' achieves the highest accuracy on the ogbn-arxiv dataset. These findings indicate that the proposed breaking term is effective in enhancing performance for node classification tasks.
    
\subsection{Discussion on Killing Process}
    In this part, we provide a detailed discussion of the killing process introduced in Section \ref{sec.5.3}. According to Theorem \ref{thm5.3}, the killing process follows the ODE dynamics described as: $\frac{\partial H(t)}{\partial t}= [A(H(t))-I]H(t)-c(H(t))H(t)$, from which the following observations can be made: 
    (1) Unlike the breaking term introduced in Section \ref{sec.4.3}, here $c(H(t))$ is negative (element-wise in the vector sense) and can depend solely on the node features $H(t)$, independent of the graph structure; (2) From the proof provided in Appendix \ref{sec.B3}, the node feature under the killing process can be expressed as $h(X_t)=\mathbb{E}^u[f(X_t)e^{-\int_0^tc(X_s)\mathbf{d}s}]$. This highlights that the parameter $c(H(t))$ in the killing process directly controls the convergence rate of the node features. 
    To more clearly illustrate the impact of the design of $c(H(t))$ on oversmoothing, we consider a special case where c is a constant. In this scenario, the node features are given by $h(X_t)=\mathbb{E}^u[f(X_t)e^{-ct}]$. Under this setting, nodes undergo oversmoothing at an exponential rate, with the rate controlled by the magnitude of $c$. We employ the log-scale Dirichlet energy to represent the extent of oversmoothing in diffusion-based GNNs as the number of layers increases. The results are presented in Table \ref{tab: killing process}.
\begin{table}[h]
\centering
\caption{Dirichlet energy in log scale as layers increase under different $c$ values.}
\label{tab: killing process}
\resizebox{.45\textwidth}{!}{%
\begin{tabular}{l c c c c c c c}
\toprule
 & \multicolumn{7}{c}{Layers} \\
\cmidrule{2-8}
$$c$$ & 1 & 10 & 20 & 60 & 80 & 90 & 100 \\
\midrule
1e-3      & 4.66 & -0.80 & -2.15 & -4.39 & -5.09 & -5.39 & -5.66 \\
1e-2      & 4.66 & -0.96 & -2.49 & -5.47 & -6.52 & -7.00 & -7.46 \\
1e-1      & 4.66 & -2.60 & -5.95 & -16.19 & -20.73 & -21.06 & -24.78 \\
5e-1      & 4.66 & -9.87 & -21.23 & -21.33 & -21.96 & -26.70 & -23.96 \\
1         & 4.66 & -18.96 & -25.00 & -22.24 & -21.43 & -23.00 & -21.63 \\
\bottomrule
\end{tabular}
}
\end{table}
It can be observed that as $c$ increases, the rate of oversmoothing accelerates significantly. In practice, the gating mechanism introduced in \cite{rusch2022gradient} can be regarded as a learning-based approach to determine $c(H(t))$, thereby aiming to alleviate the problem of oversmoothing. 
We speculate that designing different forms of $c(H(t))$ can lead to varying effects both in theory and practice.
Our work aims solely to establish this connection, while the more theoretical design of $c(H(t))$ is left as future work.

\subsection{Effects of Truncated Orders $N$}
For two datasets with distinct homophily ratios, Texas and Cora, we present in Table \ref{tab: ablation N} the impact of different Truncated Orders on the performance of the node classification task.
Results reveal that the performance is significantly impacted by the interplay between homophily and $N$. In low-homophily settings (Texas), performance is more sensitive to $N$, with Ours-Expn exhibiting superior robustness, likely due to its ability to filter out less relevant information from multi-hop aggregation. Conversely, in high-homophily settings (Cora), performance is less sensitive to changes in $N$, as immediate neighborhood information is often sufficient for classification, with all models exhibiting similar performance across values of $N$. These findings highlight the necessity of tailoring the truncation order $N$ to the dataset's homophily level to optimize GNN performance.
\begin{table}[h]
\centering
\caption{Node Classification Results w.r.t. Truncated Orders}
\label{tab: ablation N}
\resizebox{.45\textwidth}{!}{%
\begin{tabular}{l *{6}{c}}
\toprule
 & \multicolumn{3}{c}{Texas} & \multicolumn{3}{c}{Cora} \\
\cmidrule(lr){2-4} \cmidrule(lr){5-7}
Truncated Orders & N=1 & N=2 & N=3 & N=1 & N=2 & N=3 \\
\midrule
Ours-Exp & 84.06 & 79.73 & 80.27 & 87.04 & 87.97 & 87.30 \\
Ours-Expn & 88.65 & 85.41 & 87.03 & 88.31 & 88.07 & 87.85 \\
Ours-Log & 86.49 & 87.03 & 85.95 & 87.40 & 88.07 & 87.44 \\
\bottomrule
\end{tabular}
}
\end{table}

\section{Conclusions and Limitations}
This paper presents a unified framework rooted in operator semigroup theory to address oversmoothing in diffusion-based GNNs. Grounded in this framework, the introduced versatile ergodicity-breaking condition can incorporate previous research and offer universal guidance for countering oversmoothing. Empirical validation demonstrates its effectiveness in enhancing node classification performance. Our probabilistic interpretation establishes a vital connection with existing literature and enhances the comprehensiveness of our approach. 

While we discuss the nonlinear graph diffusion operator and show that attention-based nonlinear graph diffusion still suffers from oversmoothing, quantitative analysis in the continuous case is still difficult. 
In our future work, we aim to investigate the specific designs of the breaking term and the killing process to better understand their precise effects on node-level features, and will explore oversmoothing in the nonlinear case using advanced mathematical tools such as nonlinear operator semigroups.
Moreover, our experimental results primarily focus on node classification tasks, which are well-suited for evaluating whether a model is influenced by oversmoothing. Exploring the impact of the oversmoothing problem on other types of tasks, such as graph-level tasks, could serve as a promising direction for future research.


\section*{Acknowledgments}
This work is supported by National Key R\&D Program of China (No.2021YFA1000403), National Natural Science Foundation of China (No.12401666, 11991022, U23B2012, 12326611) and the Fundamental Research Funds for the Central Universities, Nankai University (No.054-63241437).
\bibliographystyle{ACM-Reference-Format}
\bibliography{ref}
\appendix
\section{Proofs for Sec. \ref{sec.3} and \ref{sec.4}} \label{appendix.proofs}
\subsection{Proof of Proposition \ref{prop.1}}\label{appendix.3.2}
    Semigroup $P_t$ induced by the operator $\mathcal{A}$ satisfies:
         $\frac{\partial P_t}{\partial t} = \mathcal{A}P_t = P_t\mathcal{A},$
         then it can be expressed as:
        $
         P_t=e^{t \mathcal{A}}.
         $
         For $f: \mathcal{V} \rightarrow \mathcal{R}^n$, we define $Q=A-I$, then $P_t f=e^{t \mathcal{A}} f= e^{tQ} f$. Since A is normalized, $Q=(q_{u v})_{u, v \in \mathcal{V}}$ is a Q-matrix satisfying the following conditions:
             (i)
              $0 \leq-q_{i i}<\infty$ for all $i$;
             (ii)
             $q_{i j} \geq 0$ for all $i \neq j$;
             (iii)
              $\sum_{j \in I} q_{i j}=0$ for all $i$.
         We state a lemma without proof, which describes the relationship between $Q$ and operator $P_t$:
         \begin{lemma}[Theorem 2.1.2 of \cite{norris1998markov} ]\label{lemma.Q matrix}
             A matrix $Q$ on a finite set is a $Q$-matrix if and only if $P_t=e^{t Q}$ is a stochastic matrix for all $t \geq 0$.
         \end{lemma}
         Lemma \ref{lemma.Q matrix} demonstrated that $P_t$ is a stochastic matrix adhere to: $0 \leq p_{i j}<\infty$ , $\sum_{j \in I} p_{i j}=1 $ for all $i$, leading to (i) and (ii). 


\subsection{Proof of Theorem \ref{thm:ergo_graph}}\label{appendix.4.3}
    \begin{lemma}\label{lemmaA1}
            Measure $\mu$ with respect to operator $\mathcal{A}$ is an positive eigenvector of the transpose of $A=(a(u,v))_{u,v \in \mathcal{V}}$ with eigenvalue 1. 
    \end{lemma}
    
    \begin{proof}
        Suppose $\mu$ is an eigenvector of $A^{T}$ with eigenvalue 1, then $ A^{T} \mu=\mu $ and  $\mu(v)=\sum_{u} a(u,v) \mu(u) $,
                    \begin{align*}
            \sum_{u \in \mathcal{V}} \mathcal{A} f(u) \mu(u) 
            &= \sum_{u} \sum_{v} a(u,v)(f(v)-f(u)) \mu(u) \\
            &=\sum_v \sum_u a(u,v) \mu(u) f(v) - \sum_u f(u) \mu(u) \\
            &=\sum_v f(v) \mu(v) - \sum_u f(u) \mu(u)=0 .
        \end{align*}

        This proves $\mu$ is invariant with respect to operator $\mathcal{A}$. And positivity of $\mu$ is ensured by the Perron-Frobenius theorem.
    \end{proof}
    

We first consider operator $ \mathcal{A} $ define in Eq.\eqref{eq.2} is self-adjoint. 
To be specific, $ \langle \mathcal{A}f, g \rangle= \langle f, \mathcal{A}g \rangle$, for every $f, g\in C(\mathcal{V})$, where inner product $ \langle f,g \rangle:=\int_{\mathcal{V}}fg \mathrm{d}\mu $.

	\begin{proof} [Proof of Theorem \ref{thm:ergo_graph}]
		Since $ \mathcal{G} $ is a connected graph, then $ A=(a(u,v))_{u,v\in\mathcal{V}} $ is irreducible, that is, for every pair of vertices $ (u,v)\in\mathcal{V}\times\mathcal{V} $, there exists a path $ (u=u_{0},u_{1},\ldots,u_{k-1},u_{k}=v) $ in $ \mathcal{V} $, such that for $ i=0,1,\ldots,k-1 $, $ a(u_{i},u_{i+1})>0 $. Equivalently, $ A $ is not similar via a permutation to a block upper triangular matrix. 
		
		From Perron-Frobenius theorem for irreducible matrices, we know that the maximal eigenvalue $ \lambda_{0}=0 $ of $ \mathcal{A} $ is unique (algebra multiplicity is $ 1 $) and the geometric multiplicity of $ \lambda_{0} $ is $ 1 $, that is, the corresponding normalized eigenfunction of $ \lambda_{0} $ is unique, noted as $ f $. Consider
		\begin{align*}
			\langle \mathcal{A}f,f\rangle=\sum_{u\in\mathcal{V}}&\left[\sum_{v\in\mathcal{V}}a(u,v)(f(v)-f(u))\right]f(u)\mu(u)\\
			=\frac{1}{2}[\sum_{u\in\mathcal{V}}&\sum_{v\in\mathcal{V}}a(u,v)(f(v)-f(u))f(u)\mu(u)\\&+\sum_{v\in\mathcal{V}}\sum_{u\in\mathcal{V}}a(v,u)(f(u)-f(v))f(v)\mu(v)]\\
			=-\frac{1}{2}&\sum_{u\in\mathcal{V}}\sum_{v\in\mathcal{V}}a(u,v)(f(u)-f(v))^{2}\mu(u),
		\end{align*}
		$ f $ is the corresponding normalized eigenfunction of $ \lambda_{0}=0 $, that is, $ \mathcal{A}f=0 $, then $ \langle \mathcal{A}f,f\rangle = 0 $. Therefore
		 $f(u)=f(v):=c,\;\forall u,v\in\mathcal{V}. $
		Notice the uniqueness of $ f $, operator $ \mathcal{A} $ is ergodic.
	\end{proof}

		

 \subsection{Proof of Theorem \ref{thm.graph_ergodicity_conv}}\label{appendix.4.4}
 \begin{proof}
			Since operator $ \mathcal{A} $ is self-adjoint, consider the spectral decomposition of the generator $\mathcal{A}$
			$$ \mathcal{A}f=\sum_{k=0}^{N}\lambda_{k}E_{\lambda_{k}}f=\sum_{k=0}^{N}\lambda_{k}f_{\lambda_{k}}e_{\lambda_{k}}, $$
			where $ \lambda_{k},k=0,1,\ldots,N $ are the eigenvalues of $ \mathcal{A} $, $ e_{\lambda_{k}} $ are corresponding normalized eigenfunctions of $ \lambda_{k} $, $ E_{\lambda_{k}} $ are corresponding projection operators of $ \lambda_{k} $, and $f_{\lambda_k}$ is the projection of $f$ related to $e_{\lambda_k}$ . For $ t\ge 0 $,
			$ P_{t}f=e^{t\mathcal{A}}f=\sum_{k=0}^{N}e^{\lambda_{k}t}f_{\lambda_{k}}e_{\lambda_{k}}. $
			Since graph $ \mathcal{G} $ is connected, Perron-Frobenius theorem implies that eigenvalues of $ \mathcal{A} $ satisfy $ 0=\lambda_{0}>\lambda_{1},\ldots,\lambda_{N} $. As $ t\rightarrow\infty $,
			$ P_tf(u)\rightarrow f_{\lambda_{0}}e_{\lambda_{0}}. $
			Since operator $ \mathcal{A} $ is ergodic, $ f_{\lambda_{0}}e_{\lambda_{0}} $ is a constant. Notice 
			$$ \mathcal{A}(\sum_{u\in \mathcal{V}}f(u)\mu(u))=\sum_{u\in \mathcal{V}}\mathcal{A}f(u)\mu(u)=0, $$
			therefore $ f_{\lambda_{0}}e_{\lambda_{0}}=\sum_{u\in \mathcal{V}}f(u)\mu(u). $
		\end{proof}
Since $ P_{t}f $ converges to the eigenspace corresponding to the eigenvalue $ \lambda_{0}=0 $, once the generator is not ergodic, the projection of function $ f\in\mathscr{L}^{2}(\mathcal{V},\mu) $ onto the eigenspace corresponding to the eigenvalue $ 0 $ is not constant. Therefore, $ \lim\limits_{t\rightarrow\infty}P_{t}f $ is not constant. That means not oversmoothing.

 If $ A $ is not symmetric, while for every pair of  vertices $ (u,v)\in\mathcal{V}\times\mathcal{V} $, $ \mu(u)a(u,v)=\mu(v)a(v,u), $ which is so called detailed balance condition, we say $ \mu $ is reversible. We can simply symmetrize $ A $ by $ \tilde{a}(u,v):=\sqrt{\mu(u)}a(u,v)\frac{1}{\sqrt{\mu(v)}}. $ Consider the symmertrized operator $ \tilde{\mathcal{A}}$ satisfying : $\tilde{\mathcal{A}} f(u)=\sum_{v\in\mathcal{V}}\tilde{a}(u,v)(f(v)-f(u)), $ we can obtain similar conclusions.
 
	 

\subsection{Proofs for Theorem \ref{thm_nonlinear}}\label{sec.nonlinear} 
    \begin{proof}

    For simplicity, we denote the attention matrix $A(H(t))$ defined in \eqref{attention_matrix} as $A^{(t)}$, satisfies
    $$
    \begin{cases}
        A^{(t)}(u,v)=\frac{e^{b^{(t)}_{u,v}}}{\sum_{w \in \mathcal{N}(u)} e^{b^{(t)}_{u,w}}} & \text{if} (u,v) \in \mathcal{E}, \\
        A^{(t)}(u,v)=0 & \text{if} (u,v) \notin \mathcal{E}.
    \end{cases}
    $$
    where $b^{(t)}_{u,v}= \Phi(h_t(u), h_t(v)),\;(u,v) \in \mathcal{E}$. Through the definition of $H_{t+1}$, we have
    $
    \|H_{t+1}\|_{max} \leq  \|A^{(t)} A^{(t-1)}... A^{(0)} \|_{max} \|H_0\|_{max}
    $, 
    where $\| \cdot \|_{max}$ denotes the matrix max norm, i.e., let $E=(e_{ij}) \in \mathbb{R}^{m \times n}$, $\|E\|_{max}=\max\limits_{i, j} e_{ij}$. Since $A^{(t)}(u,v) \in [0,1]$ for all $t$ and all $(u,v)$, 
    $\|H_{t+1}\|_{max} \leq \|H_0\|_{max}$
    Hence there exist $C_0 \geq 0$ such that for $t \in \mathbb{N}$,
    $
    \|H_{t+1}\|_{max} \leq C_0.
    $
    The continuous assumption of attention function $\Phi$ implies there exist $C \geq 0$ such that $0 \leq b^{(t)}_{u,v} \leq C$ for all $t$, we thus obtain that
    $$
    A^{(t)}(u,v)=\frac{e^{b^{(t)}_{u,v}}}{\sum_{w \in \mathcal{N}(u)} e^{b^{(t)}_{u,w}}} \geq \frac{1}{Ne^c},
    $$
    for all $(u,v) \in \mathcal{E}$.
    Hence there exist $\epsilon >0$ such that $A^{(t)}(u,v) \geq \epsilon$ for all $t \in \mathbb{N}$ and $(u,v) \in \mathcal{E}$, impling that $A^{(t)}(u,v)$ satisfies the 
    Doeblin's condition \cite{saloff2009merging}, yielding 
    $$
    max\{\|P_t(u,\cdot)-P_t(v,\cdot)\|_{TV}\} \leq (1-\epsilon)^t \rightarrow 0,\quad t \rightarrow \infty
    $$
    completing the proof.
\end{proof}

    From Theorem \ref{thm_nonlinear} we can derive another quantitative lower bounds for $P_t$ using mathematical Induction. To be specific, we aim at proving the following statement: for all $t \geq 0$ and all $(u,v) \in \mathcal{E}$, $P_{t}(u,v) \geq \epsilon^{t+1}$. When $t=0$, $P_0=A^{(0)}$, it is obvious to see the statement holds. We assume the statement holds at step $t-1$, implying that $P_{t-1}(u,v) \geq \epsilon^{t} $. The self-loop assumption indicate that
    $$
    P_{t}(u,v) \geq A^{(t)}(u,u) P_{t-1}(u,v) 
    $$
    Since $A^{(t)}(u,v) \geq \epsilon$, we have $
    P_{t}(u,v) \geq \epsilon^{t+1}, \text{for all (u,v)}
    $. 

 \subsection{Proof of Theorem \ref{theorem.extra term}}\label{appendix.4.5}
	\begin{proof}
		Let $ f_{c}\in\mathscr{L}^{2}(\mathcal{V},\mu) $ such that 
		$$ (\mathcal{A}+\mathcal{C})f_{c}(u)=\sum_{v\in\mathcal{V}}a(u,v)(f_{c}(v)-f_{c}(u))+\mathcal{C}f_{c}(u)=0. $$
		We assume that $ \mathcal{A}+\mathcal{C} $ is ergodic, that is, $ f_{c} $ is a constant. Let $ f_{c}(u)=b $, for all $ u\in\mathcal{V} $. Then
		$$ (\mathcal{A}+c)f_{c}(u)=\sum_{v\in\mathcal{V}}a(u,v)(b-b)+\mathcal{C}b=\mathcal{C}b. $$
		Since $ \mathcal{A}b=\sum_{v\in\mathcal{V}}a(u,v)(b-b)=0 $, $ \mathcal{C}b\neq 0 $. It is contradict to $ (\mathcal{A}+\mathcal{C})f_{c}(u)=0 $. Therefore $ \mathcal{A}+\mathcal{C} $ is not ergodic.
		
	\end{proof}

 \section{Proofs for Sec. \ref{sec.5}} \label{appendix.5}

 \subsection{Proof of Theorem \ref{thm5.1}}\label{sec.B1}
	\begin{proof}
		From the definition of generator,
		\begin{align*}
			\mathcal{A}h_{i}(u,t)
			&=\lim\limits_{r\rightarrow 0^{+}}\frac{\mathbb{E}^{u}[h_{i}(\mathbf{X}_{r},t)]-h_{i}(u,t)}{r}\\
			&=\lim\limits_{r\rightarrow 0^{+}}\frac{\mathbb{E}^{u}[\mathbb{E}^{\mathbf{X}_{r}}[f_i(\mathbf{X}_t)]]-\mathbb{E}^{u}[f_i(\mathbf{X}_t)]}{r}.
		\end{align*}
		Let $ \{\mathscr{F}_{t}\} $ be the right continuous $ \sigma $-algebra filtration with respect to which $ \mathbf{X}_t $ is adapted. From Chapman–Kolmogorov equation,
		$$ \mathbb{E}^{\mathbf{X}_{r}}[f_i(\mathbf{X}_t)]=\mathbb{E}^{u}[f_i(\mathbf{X}_{t+r})|\mathscr{F}_{r}]. $$
		Then,
		\begin{align*}
		\mathcal{A}h_{i}(u,t)&=\lim\limits_{r\rightarrow 0^{+}}\frac{\mathbb{E}^{u}[\mathbb{E}^{u}[f_i(\mathbf{X}_{t+r})|\mathscr{F}_{r}]]-\mathbb{E}^{u}[f_i(\mathbf{X}_t)]}{r}\\
		&=\lim\limits_{r\rightarrow 0^{+}}\frac{\mathbb{E}^{u}[f_i(\mathbf{X}_{t+r})]-\mathbb{E}^{u}[f_i(\mathbf{X}_t)]}{r}\\
		&=\lim\limits_{r\rightarrow 0^{+}}\frac{h_{i}(u,t+r)-h_{i}(u,t)}{r}=\frac{\partial h_{i}(u,t)}{\partial t}.
		\end{align*}
		Further, $ \mathbb{E}^{u}[f_i(\mathbf{X}_0)]=f_{i}(u)=h_{i}(u,0) $. Therefore, $ \mathbb{E}^{u}[f_i(\mathbf{X}_t)] $ is the solution of Cauchy problem (\eqref{eq.4}).
	\end{proof}

\subsection{Proof of Theorem \ref{thm5.2}}\label{sec.B2}
\begin{proof}
	Fix $ h > 0 $ and consider the h-skeleton $ \mathbf{Z}_{n}:=\mathbf{X}_{nh} $ of $\mathbf{X}  $. Since
	$ \mathbb{P}(\mathbf{Z}_{n+1}=u_{n+1}|\mathbf{Z}_{0}=u_{0},\ldots,\mathbf{Z}_{n}=u_{n})=p(h,u_{n},u_{n+1}), $
	h-skeleton chain $ \mathbf{Z}_{n} $ is discrete time Markov chain with transition matrix $ P(h) $. Since $ \mathcal{G} $ is a finite connected graph, for all $ u,v\in\mathcal{V} $, there exists some $ t $, such that $ p(t,u,v)>0 $. Therefore  $ \mathbf{Z}_{n} $ is irreducible and aperiodic. From classic result of discrete time Markov chain\cite{norris1998markov}, there exists a unique invariant measure $ \mu $, for all $ u,v\in\mathcal{V} $
	$$ p(nh,u,v)\rightarrow \mu(v) ,\quad n\rightarrow\infty. $$
	For fixed state $ u $, 
	\begin{align*}
		&|p(t+h,u,v)-p(t,u,v)|\\&=|\sum_{k \in \mathcal{V}}p(h,u,k)p(t,k,v)-p(t,u,v)|\\
		&=|\sum_{k \neq u}p(h,u,k)p(t,k,v)-(1-p(h,u,u))p(t,u,v)|\\
		&\leq 1-p(h,u,u)\leq \mathbb{P}(J_{1}\leq h|\mathbf{X}_{0}=u)=1-e^{h\cdot(a(u,u)-1)}.
	\end{align*}
	so given $ \epsilon>0 $ we can find $ h>0 $ such that
	$$ |p(t+h,u,v)-p(t,u,v)|\leq 1-e^{s\cdot(a(u,u)-1)}< \frac{\epsilon}{2},\quad 0\leq s\leq h $$
	and find $ N\in\mathbb{N} $ such that
	$ |p(nh,u,v)-\mu(v)|< \frac{\epsilon}{2}, n>N. $
	For $ t>Nh $ we have $ nh \leq t < (n + 1)h $ for some $ n > N $ and
	$$ |p(t,u,v)-\mu(v)|\leq|p(t,u,v)-p(nh,u,v)|+|p(nh,u,v)-\mu(v)|<\epsilon. $$
	Hence
	$p(t,u,v)\rightarrow \mu(v),\quad t\rightarrow\infty. $
\end{proof}

\subsection{Proof of Theorem \ref{thm5.3}}\label{sec.B3}
\begin{proof}
	We first show that $ \tilde{\mathbf{X}} $ is a Markov process. For all $ u\in\mathcal{V} $,
	\begin{align*}
		&\mathbb{P}(\tilde{\mathbf{X}}_{t+s}=u\;|\;\mathscr{F}_{s})=\mathbb{P}(\mathbf{X}_{t+s}=u,\tau>t+s\;|\;\mathscr{F}_{s})\\
		&=\mathbb{P}(\mathbf{X}_{t+s}=u,\tau>t+s\;|\;\mathbf{X}_{s})\mathbf{1}_{\tilde{\mathbf{X}}_{s}\neq\partial}=\mathbb{P}(\tilde{\mathbf{X}}_{t+s}=u\;|\;\tilde{\mathbf{X}}_{s})\mathbf{1}_{\tilde{\mathbf{X}}_{s}\neq\partial}.
	\end{align*}
	On the other hand, 
	\begin{align*}
	&\mathbb{P}(\tilde{\mathbf{X}}_{t+s}=\partial\;|\;\mathscr{F}_{s})=\mathbb{P}(\tau\leq t+s\;|\;\mathscr{F}_{s})\\
	&=\mathbb{P}(\tau\leq s\;|\;\mathscr{F}_{s})+\mathbb{P}(s<\tau\leq t+s\;|\;\mathscr{F}_{s})\\
	&=\mathbf{1}_{\tau\leq s}+\mathbb{P}(s<\tau\leq t+s\;|\;\mathscr{F}_{s})\\
	&=\mathbf{1}_{\tilde{\mathbf{X}}_{s}=\partial}+\mathbb{P}(s<\tau\leq t+s\;|\;\tilde{\mathbf{X}}_{s})\mathbf{1}_{\tilde{\mathbf{X}}_{s}\neq\partial}\\
	&=\mathbf{1}_{\tilde{\mathbf{X}}_{s}=\partial}+\mathbb{P}(\tilde{\mathbf{X}}_{t+s}=\partial\;|\;\tilde{\mathbf{X}}_{s})\mathbf{1}_{\tilde{\mathbf{X}}_{s}\neq\partial}=\mathbb{P}(\tilde{\mathbf{X}}_{t+s}=\partial\;|\;\tilde{\mathbf{X}}_{s}).
	\end{align*}
	Thus $ \tilde{\mathbf{X}} $ is a Markov process. Moreover, since $ \mathcal{V} $ is finite, $ \tilde{\mathbf{X}} $ is a strong Markov process. We next calculate the generators of $ \tilde{\mathbf{X}} $. We need the following lemma:
	\begin{lemma}[The Feynman-Kac formula, Theorem 8.2.1 of \citet{oksendal2013stochastic}]\label{Feynman-Kac}
	Let $ f\in C_{0}^{2}(\mathbb{R}^{n}) $ and $ q\in C(\mathbb{R}^{n}) $. Assume that $ q $ is lower bounded. Put
	$ v(t,x)=\mathbb{E}^{x}\left[\exp\left(-\int_{0}^{t}q(\xi_{s})\mathrm{d}s\right)f(\xi_{t})\right]. $
	Then
	$ \frac{\partial v}{\partial t}=Av-qv,\quad v(0,x)=f(x),\quad t>0,\;x\in\mathbb{R}^{n}, $
	where $ A $ is generator of diffusion process $ \xi $.
	\end{lemma}
	For bounded continuous functions $ f $,
	$$ \mathbb{E}^{u}[f(\tilde{\mathbf{X}}_{t})]=\mathbb{E}^{u}[f(\mathbf{X}_{t})\cdot\mathbf{1}_{[0,\tau)}(t)]=\mathbb{E}^{u}[f(\mathbf{X}_{t})\cdot e^{\int_{0}^{t}c(\mathbf{X}_{s})\mathrm{d}s}], $$
	where $ c(u)=-\lim\limits_{t\rightarrow 0^{+}}\frac{1}{t}\mathbb{P}(\tau\leq t|\mathbf{X}_{0}=u) $. Then, by lemma \ref{Feynman-Kac},
	$$ \lim_{t\rightarrow 0^{+}}\frac{\mathbb{E}^{u}[f(\tilde{\mathbf{X}}_{t})]-f(u)}{t}=\mathcal{A}f(u)+c(u)f(u). $$
	 That is, the generator of $ \tilde{\mathbf{X}} $ is
	$ \tilde{\mathcal{A}}=\mathcal{A}+c(u). $
\end{proof}

\section{Experimental Details} \label{appendix.exp}
 \subsection{Configurations for Oversmoothing}\label{appendix. os}
 \textbf{Synthetic Cora Dataset.} The Synthetic Cora dataset, as introduced by \citet{zhu2020beyond}, is a variant of the well-known Cora dataset tailored for graph neural network (GNN) experiments. This dataset is characterized by its homophily index, which ranges from 0 to 1. 

 \noindent\textbf{Configurations For Figure 1.} For all ergodicity-breaking terms, we use a first-order truncation for ergodicity-breaking terms within our proposed models: Ours-Exp, Ours-ExpN and Ours-Log. 

  \subsection{Configurations for Node Classification Tasks}\label{appendix. nct}
    All experiments reported in this work were conducted within a uniform computational environment to ensure reproducibility and consistency of results. The specifications of the environment are as follows:
    Operating System: Ubuntu 18.04 LTS;
    Programming Language: Python 3.10.4;
    Deep Learning Framework: PyTorch 2.0.1;
    Graph Neural Network Library: PyTorch Geometric 2.4.0;
    Differential Equation Solver: TorchDiffEq 0.2.3;
    GPU Computing: CUDA 11.7;
    Processor: AMD EPYC 7542 32-Core Processor;
    Graphics Card: NVIDIA RTX 3090.

    The optimal hyperparameters for each term were meticulously tuned based on the settings in GREAD \cite{choi2023gread}.
    The ergodicity-breaking terms introduced in Sec. \ref{ncexp} were configured with specific truncated orders N. 
    Empirically, we observe that the hyperparameters specified in GREAD demonstrate robust experimental performance. 
    The code repository for the paper is available at: \href{https://github.com/LOGO-CUHKSZ/SGOS}{https://github.com/LOGO-CUHKSZ/SGOS}

\end{document}